\documentclass[10pt, twoside]{article}
\usepackage[margin=1in]{geometry}

\usepackage{subfig}

\pdfoutput=1

\usepackage{amsfonts,amsmath,amssymb,amsthm,bm,bbm}

\usepackage{algorithm,algpseudocode}
\usepackage{wrapfig}
\usepackage{dsfont,array}

\usepackage{newclude}

\usepackage{appendix}

\usepackage[colorlinks,
linkcolor=blue,
citecolor=blue,
urlcolor=magenta]{hyperref}

\usepackage{nicefrac,comment}
\usepackage{multirow}
\usepackage{multicol}
\usepackage{xspace}
\usepackage{mathtools}
\usepackage{placeins}
\usepackage{authblk}

\newtheorem{definition}{Definition}
\newtheorem{assumption}{Assumption}
\newtheorem{theorem}{Theorem}
\newtheorem{lemma}{Lemma}
\newtheorem{corollary}{Corollary}

\newtheorem{remark}{Remark}

 \newcommand{\ignore}[1]{}

\DeclarePairedDelimiter\floor{\lfloor}{\rfloor}
\DeclareMathOperator*{\argmax}{arg\,max}
\newcommand\numberthis{\addtocounter{equation}{1}\tag{\theequation}}

\DeclarePairedDelimiterX\Basics[1](){ #1}

\usepackage{centernot}

\newcounter{const-no}

\def\EE{{\mathbb{E}}}\def\PP{{\mathbb{P}}}
\def\cF{{\mathcal{F}}}

\DeclareMathOperator*{\argmin}{\arg\!\min}

\makeatletter
\newcommand{\setword}[2]{%
	\phantomsection
	#1\def\@currentlabel{\unexpanded{#1}}\label{#2}%
}
\makeatother

\title{Contextual Bandits with Stochastic Experts}

\author[1]{Rajat Sen}
\author[2]{Karthikeyan Shanmugam}
\author[1]{Nihal Sharma}
\author[1]{Sanjay Shakkottai}
\affil[1]{The University of Texas at Austin}
\affil[2]{IBM Research, Thomas J. Watson Center}

\usepackage[parfill]{parskip}

\begin{document}
	\maketitle

%

%

%
%
%

\begin{abstract}
  We consider the problem of contextual bandits with stochastic experts, which is a variation of the traditional stochastic contextual bandit with experts problem. In our problem setting, we assume access to a class of {\it stochastic experts}, where each expert is a conditional distribution over the arms given a context. We propose upper-confidence bound (UCB) algorithms for this problem, which employ two different importance sampling based estimators for the mean reward for each expert. Both these estimators leverage \textit{information leakage} among the experts, thus using samples collected under all the experts to estimate the mean reward of any given expert. This leads to \textit{instance dependent} regret bounds of $\mathcal{O}\left(\lambda(\pmb{\mu})\mathcal{M}\log T/\Delta \right)$, where $\lambda(\pmb{\mu})$ is a term that depends on the mean rewards of the experts, $\Delta$ is the smallest gap between the mean reward of the optimal expert and the rest, and $\mathcal{M}$ quantifies the information leakage among the experts. We show that under some assumptions $\lambda(\pmb{\mu})$ is typically $\mathcal{O}(\log N)$. We implement our algorithm with stochastic experts generated from cost-sensitive classification oracles and show superior empirical performance on real-world datasets, when compared to other state of the art contextual bandit algorithms.  
\end{abstract}

\section{Introduction}
\label{sec:intro}
Modern machine learning applications like recommendation engines~\cite{li2011scene,bouneffouf2012contextual,li2010contextual}, computational advertising~\cite{tang2013automatic,bottou2013counterfactual}, A/B testing in medicine~\cite{tekin2015discover,tekin2016adaptive} are inherently online. In these settings the task is to take sequential decisions that are not only profitable but also enable the system to learn better in future. For instance in a computational advertising system, the task is to sequentially place advertisements on users' webpages with the dual objective of learning the preferences of the users and increasing the click-through rate on the fly. A key attribute of these systems is the well-known \textit{exploration} (searching the space of possible decisions for better learning) and \textit{exploitation} (taking decisions that are more profitable) trade-off.\footnote{This paper is a revised version of \cite{pmlr-v84-sen18a}, where some of the concentration bounds in the Appendix had flaws. We have  updated the proofs, the corresponding constants in the Algorithms and the bounds in the Appendix. As a result, the  multiplicative constants in the regret analysis have been changed  and the simulations have been revised.}
A principled method to capture this trade-off is the study of multi-armed bandit problems~\cite{bubeck2012regret}.

$K$-armed stochastic bandit problems have been studied for several decades.
These are formulated as a sequential process, where at each time step any one of the $K$-arms can be selected. Upon selection of the $k$-th arm, the arm returns a stochastic reward with an expected reward of $\mu_k$.
%
Starting from the work of  \cite{lai1985asymptotically}, a major focus has been on \textit{regret}, which is the difference in the total reward that is accumulated from the \textit{genie} optimal policy (one that always selects the arm with the maximum expected reward) from that of the chosen online policy.
The current state-of-art algorithms achieve a regret
of $O((K/\Delta) \log T )$
\cite{bubeck2012regret,auer2010ucb,agrawal2012analysis,auer2002using},
which is order-wise optimal~\cite{lai1985asymptotically}. Here, $\Delta$ corresponds to the gap in expected reward between the best arm and the next best one.


Additional side information can be incorporated in this setting through the framework of contextual bandits. In the stochastic setting, it is assumed that at each time-step nature draws $(x,r_1,...,r_K)$ from a fixed but unknown distribution. Here, $x \in \mathcal{X}$ represents the context vector, while $r_1,...,r_K$ are the rewards of the $K$-arms~\cite{langford2008epoch}. The context $x$ is revealed to the policy-designer, after which she decides to choose an arm $a \in \{1,2,...,K \}$. Then, the reward $r_a$ is revealed to the policy-designer. In the computational advertising example, the context can be thought of as the browsing history, age, gender etc. of an user arriving in the system, while $r_1,...,r_K$ are generated according to the probability of the user clicking on each of the $K$ advertisements. The task here is to learn a \textit{good} mapping from the space of contexts $\mathcal{X}$ to the space of arms $[K] = \{1,2,...,K\}$ such that when the decisions are taken according to that mapping, the mean reward observed is high.

A popular model in the stochastic contextual bandits literature is the \textit{experts} setting~\cite{agarwal2014taming,dudik2011efficient,langford2008epoch}. The task is to compete against the best \textit{expert} in a class of experts $\Pi = \{\pi_1,...,\pi_N\}$, where each expert $\pi \in \Pi$ is a function mapping $\mathcal{X} \rightarrow [K]$. The mean reward of an expert $\pi$ is defined as $ \EE \left[ r_{\pi(X)} \right]$, where $X$ is the random variable denoting the context and the expectation is taken over the unknown distribution over $(x,r_1,...,r_K)$. The best expert is naturally defined as the expert with the highest mean reward. The expected difference in rewards of a genie policy that always chooses the best expert and the online algorithm employed by the policy-designer is defined as the regret. This problem has been well-studied in the literature, where a popular approach is to reduce the contextual bandit problem to supervised learning techniques through $\argmin$-oracles~\cite{beygelzimer2009offset}. This leads to powerful algorithms with instance-independent regret bounds of $\mathcal{O} \left(\sqrt{KT\mathrm{polylog}(N)} \right)$ at time $T$~\cite{agarwal2014taming,dudik2011efficient}.  

In practice the class of experts are generated online by training cost-sensitive classification oracles~\cite{agarwal2014taming,dudik2011efficient}. Once trained, the resulting classifiers/oracles can provide reliable confidence scores given a new context, especially if they are well-calibrated~\cite{cohen2004properties}. These confidence scores effectively are a $K$-dimensional probability vector, where the $k^{th}$ entry is the probability of the classifier/oracle choosing the $k^{th}$ arm as the best, given a context. Motivated by this observation, we propose a variation of the traditional experts setting, which we term contextual bandits with \textit{stochastic experts.} We assume access to a class of \textit{stochastic experts} $\Pi = \{\pi_1,...,\pi_N\}$, which are \textit{not deterministic}. Instead, each expert $\pi \in \Pi$, is a conditional probability distribution over the arms given a context. For an expert $\pi \in \Pi$ the conditional distribution is denoted by $\pi_{V \vert X}(v \vert x)$ where $V \in [K]$ is the random variable denoting the arm chosen and $X$ is the context.  An additional benefit is that this setting allows us to derive regret bounds in terms of \textit{closeness} of these soft experts quantified by divergence measures, rather than in terms of the total number of arms $K$.

As before, the task is to compete against the expert in the class with the highest mean reward. The expected reward of a stochastic expert $\pi$ is defined as $\EE_{X,V \sim \pi(V \vert X)} \left[ r_{V}\right]$, i.e the mean reward observed when the arm is drawn from the conditional distribution $\pi(V \vert X)$. We propose upper-confidence (UCB) style algorithms for the contextual bandits with stochastic experts problem, that employ two importance sampling based estimators for the mean rewards under various experts. We prove \textit{instance-dependent} regret guarantees for our algorithms. The main contributions of this paper are listed in the next section.

\subsection{Main Contributions}
The contributions of this paper are three-fold:

{\bf $(i)$ (Importance Sampling based Estimators):}
The key components in our approach are two importance sampling based estimators for the mean rewards under all the experts. Both these estimators are based on the observation that samples collected under one expert can be reweighted by likelihood/importance ratios and averaged to provide an estimate for the mean reward under another expert. This sharing of information is termed as \textit{information leakage} and has been utilized before under various settings~\cite{lattimore2016causal,pmlr-v70-sen17a,bottou2013counterfactual}. The first estimator that we use is an adaptive variant of the well-known clipping technique, which was proposed in~\cite{pmlr-v70-sen17a}. The estimator is presented in Eq.~\eqref{eq:est1}. However, we carefully adapt the clipping threshold in an online manner, in order to achieve regret guarantees. 

We also propose an importance sampling variant of the classical median of means estimator (see ~\cite{lugosi2017sub,bubeck2013bandits}). This estimator is also designed to utilize the samples collected under all experts together to estimate the mean reward under any given expert. We define the estimator in Eq.~\eqref{eq:est2}. To the best of our knowledge, importance sampling has not been used in conjunction with the median of means technique in the literature before. We provide novel confidence guarantees for this estimator which depends on chi-square divergences between the conditional distributions under the various experts. This may be of independent interest. 

{\bf $(ii)$ (Instance Dependent Regret Bounds):}  We propose the contextual bandits with stochastic experts problem. We design two UCB based algorithms for this problem, based on the two importance sampling based estimators mentioned above. We show that utilizing the \textit{information leakage} between the experts leads to regret guarantees that scale sub-linearly in $N$, the number of experts. The information leakage between any two experts in the first estimator is governed by a pairwise log-divergence measure (Def.~\ref{def:mij}). For the second estimator, chi-square divergences (Def.~\ref{def:sij}) characterize the leakage.

We show that the regret of our UCB algorithm based on these two estimators scales as \footnote{Tighter regret bounds are derived in Theorems \ref{thm:r1} and \ref{thm:r2}. Here, we only mention the Corollaries of our approach, that are easy to state.}: $ \mathcal{O}\left(  \frac{\lambda(\pmb{\mu}) \mathcal{M}}{\Delta}  \log T  \right)$.

Here, $\mathcal{M}$ is related to the largest pairwise divergence values under the two divergence measures used. $\Delta$ is the gap between the mean rewards of the optimal expert and the second best. $\lambda(\pmb{\mu})$ is a parameter that only depends on the gaps between mean rewards of the optimum experts and various sub-optimal ones. It is a normalized sum of difference in squares of the gaps of adjacent sub-optimal experts ordered by their gaps. Under the assumption that the suboptimal gaps (except that of the second best arm) are uniformly distributed in a bounded interval, we can show that the parameter $\lambda(\pmb{\mu})$ is $O(\log N)$ in expectation. We define this parameter explicitly in Section \ref{sec:results}.

 For the clipped estimator we show that $\mathcal{M}= M^4 \log^2 (1/\Delta)$ where $M$ is the largest pairwise log-divergence associated with the clipped estimator. For the median of means estimator,  $\mathcal{M}= \sigma^4 $ where $\sigma^2$ is the largest pairwise chi-squared divergence. 
 
 Naively treating each expert as an arm would lead to a regret scaling of $\mathcal{O}(N \log T/ \Delta)$. However, this ignores information leakage. Existing instance-independent bounds for contextual bandits scale as $\sqrt{KT \mathrm{poly} \log(N) }$~\cite{agarwal2014taming}. Our problem dependent bounds have a near optimal dependence on $\Delta$ and does not depend on $K$, the numbers of arms. However, it depends on the divergence measure associated with the information leakage in the problem ($M$ or $\sigma$ parameters). Besides our analysis, we empirically show that this divergence based approach rivals or performs better than very efficient heuristics for contextual bandits (like bagging etc.) on real-world data sets.

%

{\bf $(iii)$ (Empirical Validation):} We empirically validate our algorithm on three real world data-sets~\cite{frey1991letter,horton1996probabilistic,stream} against other state of the art contextual bandit algorithms~\cite{langford2008epoch,agarwal2014taming} implemented in Vowpal Wabbit~\cite{wabbit}. In our implementation, we use online training of cost-sensitive classification oracles~\cite{beygelzimer2009offset} to generate the class of stochastic experts. We show that our algorithms have better regret performance on these data-sets compared to the other algorithms.

\section{Related Work}
\label{sec:rwork}
Contextual bandits has been studied in the literature for several decades, starting with the simple setting of discrete contexts~\cite{bubeck2012regret}, to linear contextual bandits~\cite{chu2011contextual} and finally the general experts setting~\cite{dudik2011efficient,agarwal2014taming,langford2008epoch,auer2002nonstochastic,beygelzimer2011contextual}. In this work, we focus on the experts setting. Contextual bandits with experts was first studied in the adversarial setting, where there are algorithms with the optimal regret scaling $\mathcal{O}(\sqrt{KT\log N})$~\cite{auer2002nonstochastic}. 

In this paper, we are more interested in the stochastic version of the problem, where the context and the rewards of the arms are generated from an unknown but fixed distribution. The first strategies to be explored in this setting were explore-then-commit and epsilon-greedy~\cite{langford2008epoch} style strategies that achieve a regret scaling of $\mathcal{O}\left(\sqrt{K \log N} T^{2/3} \right)$ in the instance-independent case. Following this there have been several efforts to design adaptive algorithms that achieve a $\mathcal{O}(\sqrt{KT\mathrm{polylog}(N) })$ instance-independent regret scaling. Notable among these are~\cite{dudik2011efficient,agarwal2014taming}. These algorithms map the contextual bandit problem to supervised learning and assume access to cost-sensitive classification oracles. These algorithms have been heavily optimized in Vowpal Wabbit~\cite{wabbit}. 

We study the contextual bandits with stochastic experts problem, where the experts are not deterministic functions mapping contexts to arms, but are conditional distributions over the arms given a context. We show that we can achieve instance-dependent regret guarantees for this problem, that can scale as $\mathcal{O}\left((\mathcal{M}\log N/\Delta) \log T\right)$ under some assumptions. Here, $\Delta$ is the gap between the mean reward of the best expert and the second best and $\mathcal{M}$ is a divergence term between the experts. Our algorithms are based on importance sampling based estimators which leverage information leakage among stochastic experts. We use an adaptive clipped importance sampling estimator for the mean rewards of the experts, that was introduced in~\cite{pmlr-v70-sen17a}. In~\cite{pmlr-v70-sen17a}, the estimator was studied in a best-arm/pure explore setting, while we study a cumulative regret problem where we need to adjust the parameters of the estimator in an online manner. In addition, we introduce an importance sampling based median of means style estimator in this paper, that can leverage the information leakage among experts.  
\section{Problem Setting and Definitions}
\label{sec:defs}
The general stochastic contextual bandit problem with $K$ arms is defined as a sequential process for $T$ discrete time-steps~\cite{langford2008epoch}, where $T$ is the time-horizon of interest. At each time $t \in \{ 1,2,\cdots,T\}$ nature draws a vector $(x_t, r_1(t),...,r_K(t))$ from an unknown but fixed probability distribution. Here, $r_{i}(t) \in [0,1]$ is the reward of arm $i$. The context vector $x_t \in \mathcal{X}$ is revealed to the policy-designer, whose task is then to choose an arm out the $K$ possibilities. Only the reward $r_{v(t)}(t)$ of the chosen arm $v(t)$, is then revealed to the policy-designer. We will use $r_{v(t)}$ in place of $r_{v(t)}(t)$ for notational convenience. 


\begin{figure}[h]
	\centering
	\includegraphics[width=8cm,height=10cm,keepaspectratio]{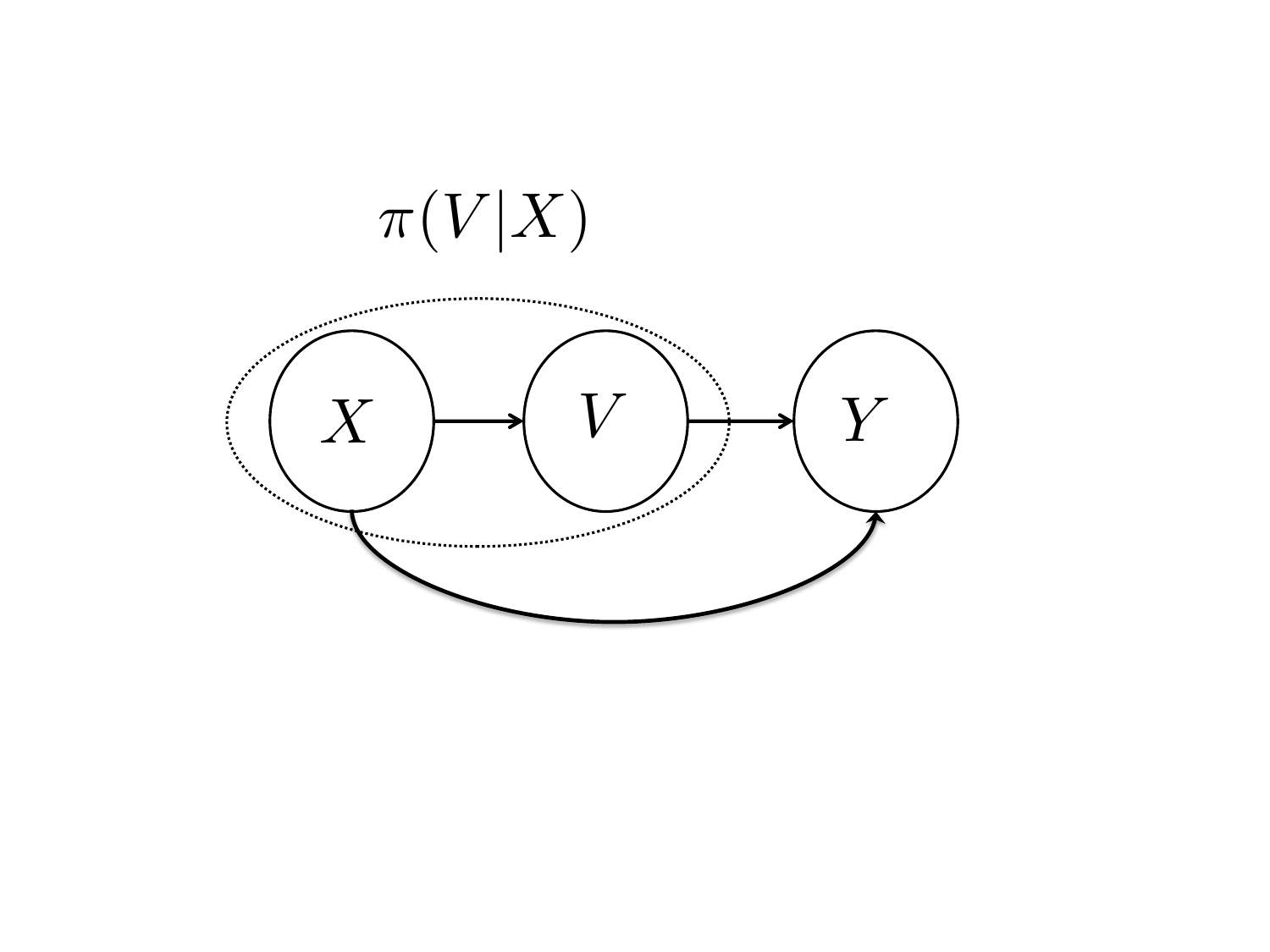}
	\caption{\small Bayesian Network denoting the joint distribution of the random variables at a given time-step, under our contextual bandit setting. $X$ denotes the context, $V$ denotes the chosen arm, while $Y$ denotes the reward from the chosen arm that also depends on the context observed. The distribution of the reward given the chosen arm and the context, and the marginal of the context remain fixed over all time slots. However, the conditional distribution of the chosen arm given the context is dependent on the stochastic expert at that time-step. }
	\label{fig:illustrate}
\end{figure}

{\bf Stochastic Experts: } 
We consider a class of stochastic experts $\Pi = \{\pi_1,\cdots, \pi_N \}$, where each $\pi_i$ is a conditional probability distribution $\pi_{V \vert X}(v \vert x)$ where $V \in [K]$ is the random variable denoting the arm chosen and $X$ is the context. We will use the shorthand $\pi_i(V \vert X)$ to denote the conditional distribution corresponding to expert $i$, for notational convenience. The observation model at each time step $t$ is as follows: {\it (i)} A context $x_t$ is observed. {\it (ii)} The policy-designer chooses a stochastic expert $\pi_{k(t)} \in \Pi$. An arm $v(t)$ is drawn from the probability distribution $\pi_{k(t)}(V \vert x_t)$, by the policy-designer.  {\it (iv)} The stochastic reward $y_t = r_{v(t)}$ is revealed. 

The joint distribution of the random variables $X,V,Y$ denoting the context, arm chosen and reward observed respectively at time $t$, can be modeled by the Bayesian Network shown in Fig.~\ref{fig:illustrate}. The joint distribution factorizes as follows, $p(x,v,y) = p(y \vert v,x ) p(v \vert x)p(x)~\refstepcounter{equation}(\theequation)\label{eq:joint}$,
where $p(y \vert v,x )$ (the reward distribution given the arm and the context), and $p(x)$ (marginal distribution of the context) is determined by the nature's distribution and are fixed for all time-steps $t = 1,2,...,T$. On the other hand $p(v \vert x)$ (distribution of the arm chosen given the context) depends on the expert selected at each round. At time $t$, $p(v \vert x) = \pi_{k(t)}(v \vert x)$ that is the conditional distribution encoded by the stochastic expert chosen at time $t$. Now we are at a position to define the objective of the problem.


{\bf Regret: } The objective in our contextual bandit problem is to perform as well as the best expert in the class of experts. We will define $p_k(x,v,y) \triangleq p(y \vert v,x ) \pi_k(v \vert x)p(x)$ as the distribution of the corresponding random variables when the expert chosen is $\pi_k \in \Pi$. The expected reward of expert $k$ is now denoted by,
 $\mu_k = \EE_{p_k(x,v,y)}[Y],$ where $\EE_{p(.)}$ denotes expectation with respect to distribution $p(.)$. The best expert is given by $k^* = \argmax_{k \in [N]} \mu_k$. The objective is to minimize the \textit{regret} till time $T$, which is defined as $R(T) = \sum_{t = 1}^{T} \EE\left[\mu^* - \mu_{k(t)} \right]$, where $\mu^* = \mu_{k^*}$. Note that this is analogous to the regret definition for the deterministic \textit{expert} setting~\cite{langford2008epoch}. Let us define $\Delta_{k} \triangleq \mu^* - \mu_k$ as the optimality gap in terms of expected reward, for expert $k$. Let $\pmb{\mu} \triangleq \{\mu_1,...,\mu_N \}$.  We further assume that for all $i \in [N]$, $\mu_i \geq \gamma$. Now we will define some divergence metrics that will be important in describing our algorithms and theoretical guarantees. 
 

 \subsection{Divergence Metrics}
 \label{sec:divergence}
 In this section we will define some $f$-divergence metrics that will be important in analyzing our estimators. Similar divergence metrics were defined in~\cite{pmlr-v70-sen17a} to analyze the clipped estimator in~\eqref{eq:est1} in the context of a best arm identification problem. In addition to the divergence metric in~\cite{pmlr-v70-sen17a}, we will also define the chi-square divergence metric which will be useful in analyzing the median of means based estimator~\eqref{eq:est2}. First, we define conditional $f$-divergence. 
 
 \begin{definition}
 	Let $f(\cdot)$ be a non-negative convex function such that $f(1) = 0$.
 	For two joint distributions $p_{X,Y}(x,y)$ and $q_{X,Y}(x,y)$ (and the
 	associated  conditionals), the conditional $f$-divergence $D_{f}(p_{X
 		\vert Y} \Vert q_{X \vert Y})$ is given by: 
 	
 	$D_{f}(p_{X \vert Y} \Vert q_{X \vert Y}) = \EE_{q_{X,Y}} \left[f \left( \frac{p_{X \vert Y}(X \vert Y)}{q_{X \vert Y}(X \vert Y)} \right)\right].$

 \end{definition}
 
 Recall that $\pi_i$ is a conditional distribution of $V$ given $X$. Thus, $D_f(\pi_i \Vert
 \pi_j)$ is the conditional $f$-divergence between the conditional distributions $\pi_i$ and $\pi_j.$ Note that in this definition the marginal distribution of $X$ is the marginal of $X$ given by nature's inherent distribution over the contexts. In this work we will be concerned with two specific $f$-divergence metrics that are defined as follows. 
 
 \begin{definition}
 	\label{def:mij}
 	($M_{ij}$ measure)~\cite{pmlr-v70-sen17a}  Consider the function $f_1(x) = x \exp (x-1) -
 	1$. We define the following log-divergence measure:  $M_{ij} = 1 +
 	\log (1 + D_{f_1} (\pi_i \lVert \pi_j)),$ $\forall i,j
 	\in [N].$ 
 \end{definition}
 
 The $M_{ij}$-measures will be crucial in analyzing one of our estimators (clipped estimator) defined in Section~\ref{sec:estimators}. 
 
 \begin{definition}
 	\label{def:sij}
 	($\sigma_{ij}$ measure) $D_{f_2} (\pi_i \Vert \pi_j)$ is known as the chi-square divergence between the respective conditional distributions, where $f_2(x) = x^2 - 1$. Let $\sigma^2_{ij} = 1 + D_{f_2} (\pi_i \Vert \pi_j)$. 
 \end{definition}

The $\sigma_{ij}$-measures are important in analyzing our second estimator (median of means) defined in Section~\ref{sec:estimators}.

\section{A Meta-Algorithm}
\label{sec:algo}
In this section, we propose a general upper-confidence bound (UCB) style strategy that utilizes the structure of the problem to converge to the best expert much faster than a naive UCB strategy that treats each expert as an arm of the bandit problem. One of the key observations in this framework is that rewards collected under one expert can give us valuable information about the mean under another expert, owing to the Bayesian Network factorization of the joint distribution of $X,V$ and $Y$. We propose two estimators for the mean rewards of different experts, that leverage this information leakage among experts, through importance sampling. These estimators are defined in Section~\ref{sec:estimators}. We propose a meta-algorithm (Algorithm~\ref{alg:dUCB}) that is designed to use these estimators and the corresponding confidence intervals, to control regret. 
\begin{algorithm}
	\caption{D-UCB: Divergence based UCB for contextual bandits with stochastic experts}
	\begin{algorithmic}[1]
		\State For time step $t = 1$, observe context $x_1$ and choose a random expert $\pi \in \Pi$. Play an arm drawn from the conditional distribution $\pi(V \vert x_1)$. 
		\For {$t = 2,...,T$}
		\State Observe context $x_t$
		\State Let $k(t) = \argmax_{k} U_{k}(t-1) \triangleq \hat{\mu}_k(t-1) + s_k(t-1)$. 
		\State \parbox[t]{\dimexpr\linewidth-\algorithmicindent}{Select an arm $v(t)$ from the distribution \\ $ \pi_{k(t)} (V \vert x_t)$.} 
		\State Observe the reward $Y(t)$.
		\EndFor
	\end{algorithmic}
	\label{alg:dUCB}
\end{algorithm}

Here, $\hat{\mu}_k(t)$ denotes an estimate for the mean reward for expert $k$ at time $t$, while $s_k(t)$ denotes the upper confidence bound for the corresponding estimator at time $t$. We propose two estimators that utilize all the samples observed under various experts to provide an estimate for the mean reward under expert $k$. 

The first estimator denoted by $\hat{\mu}_k^{c}(t)$ (Section~\ref{sec:estimators}, Eq.~\eqref{eq:est1}) is a clipped importance sampling estimator inspired by ~\cite{pmlr-v70-sen17a}. If this estimator is used, then $s_k(t)$ is set as in Equation.~\eqref{eq:ucb1}.  

The second estimator denoted by $\hat{\mu}_k^{m}(t)$ (Section~\ref{sec:estimators}, Eq.~\eqref{eq:est2}) is a median of means based importance sampling estimator. If this estimator is used, then $s_k(t)$ is set as in Equation.~\eqref{eq:ucb2}.  

\section{Estimators and Confidence Bounds}
\label{sec:estimators}

In this section we define two estimators for estimating the mean rewards under a given expert. Both these estimators can effectively leverage the information leakage between samples collected under various experts, through importance sampling. One key observation that enables us in doing so is the following equation,
\begin{align}
\label{eq:infoleakage}
\mu_k = \EE_{p_j(x,v,y)} \left[ Y \frac{\pi_k(V \vert X)}{\pi_j(V \vert X)}\right]. 
\end{align}

This has been termed as \textit{information leakage} and has been leveraged before in the literature~\cite{pmlr-v70-sen17a, lattimore2016causal,bottou2013counterfactual} in best-arm identification settings. Recall that the subscript $p_j(x,v,y)$ denotes that the expectation is taken under the joint distribution in~\eqref{eq:joint}, where $p(v \vert x) = \pi_j(v \vert x)$ i.e. under the distribution imposed by expert $\pi_j$. However, even under this distribution we can technically estimate the mean reward under expert $\pi_k$. The above equation is the motivation behind our estimators. Now, we will introduce our first estimator.  

{\bf Clipped Estimator: } This estimator was introduced in~\cite{pmlr-v70-sen17a} in the context of a pure exploration problem. Here, we analyze this estimator in a cumulative regret setting, where the parameters of the estimator need to be adjusted differently. Let $n_i(t)$ denote the number of times expert $i$ has been invoked by Algorithm~\ref{alg:dUCB} till time $t$, for all $i \in [N]$. We define the fraction $\nu_i(t) \triangleq n_i(t)/t$. We will also define $\mathcal{T}_i(t)$ as the subset of time-steps among $\{1,..,t\}$, in which the expert $i$ was selected. Let $\hat{\mu}^c_k(t)$ be the estimate of the mean reward of expert $k$ from samples collected till time $t$. The estimator is given by,

\begin{align*}
\label{eq:est1}
\hat{\mu}^c_{k}(t)  =& \frac{1}{Z_{k}(t)}\sum_{j = 1}^{N} \sum_{s \in
	\mathcal{T}_j(t)} \frac{1}{M_{kj}}Y_j(s)\frac{\pi_k(V_j(s) \vert
	X_j(s))}{\pi_j(V_j(s) \vert X_j(s))} 
\times \mathds{1}\left\{ \frac{\pi_k(V_j(s) \vert X_j(s))}{\pi_j(V_j(s) \vert X_j(s))} \leq 2\log(2/\epsilon(t))M_{kj}\right\}. \numberthis
\end{align*}

Here, $A_j(s)$ is the value of the random variable $A$ at time $s$ drawn using expert $j$, where $A$ can be the r.v's $X$,$Y$ or $V$. We set $Z_{k}(t) = \sum_{j} n_j(t)/M_{kj}$. $\epsilon(t)$ is an adjustable term which controls the bias-variance trade-off for the estimator. 

{\bf Intuition:} The clipped estimator is a weighted average of the samples collected under different experts, where each sample is scaled by the importance ratio as suggested by~\eqref{eq:infoleakage}. We also clip the importance ratios which are larger than a clipper level. This clipping introduces bias but decreases variance. The clipper level is carefully chosen to trade-off bias and variance. The clipper level values and the weights are dependent on the divergence terms $M_{kj}$'s. When the divergence $M_{kj}$ is large, it means that the samples from expert $j$ is not valuable for estimating the mean for expert $k$. Therefore, a weight of $1/M_{kj}$ is applied. Similarly, the clipper level is set at $2\log(2/\epsilon(t))M_{kj}$ to restrict the aditive bias to $\epsilon(t)$.  

The upper confidence term in Algorithm~\ref{alg:dUCB} for the estimator $\hat{\mu}^c_{k}(t)$ is chosen as,
\begin{align}
\label{eq:ucb1}
s^c_{k}(t) = \frac{3}{2}\beta(t) 
\end{align}
 at time $t$, where $\beta(t)$ is such that, $\beta(t) = C w \left(\frac{\sqrt{c_1 t \log t}}{Z_k(t)}\right).$  We set $c_1 = 1$ and $C = 16M/\gamma$ in our analysis. The function $w(\cdot)$ is defined as $w(x) = y$ s.t $y/\log(2/y) = x$.

{\bf Median of Means Estimator: } Now we will introduce our second estimator which is based on the well-known median of means technique of estimation. Median of means estimators are popular for statistical estimation when the underlying distributions are heavy-tailed~\cite{bubeck2013bandits}. The estimator for the mean under the $k^{th}$ expert at time $t$ is obtained through the following steps:  ($i$) We divide the total samples into $l(t) = \floor{ c_2 \log (1/\delta(t))}$ groups of equal size (while throwing away extra samples) where the partition respects the order in which the samples are recieved. We choose $c_2 = 8$ for our analysis. Let us index the groups as $r = 1,2...,l(t)$. ($ii$) We calculate the empirical mean of expert $k$ from the samples in each group through importance sampling. ($iii$) The median of these means is our estimator. 

Now we will setup some notation. Let $\mathcal{T}_{j}^{(r)} \subseteq \{(r-1)l(t) + 1, \cdots, rl(t) \}$ be the indices of the samples from expert $j$ that lie in group $r$. Let $W_k(r,t) = \sum_{i}n_i(r,t)/\sigma_{ki}$, where $n_i(r,t)$ is the number of samples from expert $i$ in group $r$. Let $n(r,t) = \sum_i n_i(r,t)$. Then the mean of expert $k$ estimated from group $r$ is given by,

\begin{align*}
\label{eq:momm}
&\hat{\mu}_{k}^{(r)}(t)  = \frac{1}{W_k(r,t)}\sum_{j = 1}^{N} \sum_{s \in
	\mathcal{T}^{(r)}_j} \frac{1}{\sigma_{kj}}Y_j(s)\frac{\pi_k(V_j(s) \vert
	X(s))}{\pi_j(V_j(s) \vert X_j(s))}. \numberthis
\end{align*} 

The median of means estimator for expert $k$ is then given by,
\begin{align}
\label{eq:est2}
\hat{\mu}^m_{k}(t) \triangleq  \mathrm{median} \left(\hat{\mu}_{k}^{(1)}(t), \cdots, \hat{\mu}_{k}^{(l(t))}(t) \right).
\end{align}

{\bf Intuition:} The mean of every group is a weighted average of samples from each expert, rescaled by the importance ratios. This is similar to the clipped estimator in Eq.~\eqref{eq:est1}. However, here the importance ratios are not clipped at a particular level. In this estimator, the bias-variance trade-off is controlled by taking the median of means from $l(t)$ groups. The number of groups $l(t)$ needs to be carefully set in-order to control the bias-variance trade-off.  

The upper confidence bound used in conjunction with this estimator at time $t$ is given by,

\begin{align}
\label{eq:ucb2}
s^m_k(t) = \frac{1}{W_k(t)}\sqrt{\frac{c_3 \log (1/\delta(t))}{t}}
\end{align}
where $W_k(t) = \min_{r \in [l(t)]} W_k(r,t)/n(r,t)$ and $\delta(t)$ is set as $1/t^2$ in our algorithm. We set the constant $c_3 = 32 \sigma^2$ for our analysis, where $\sigma = \max_{ij} \sigma_{ij}$.
\section{Theoretical Results}
\label{sec:results}
In this section, we provide \textit{instance dependent} regret guarantees for Algorithm~\ref{alg:dUCB} for the two estimators proposed - a) The clipped estimator~\eqref{eq:est1} and b) The median of means estimator~\eqref{eq:est2}. Let $\Delta = \min_{k \neq k^*} \Delta_k$ be the gap in the expected reward between the optimum expert and the second best. We define a parameter $\lambda(\pmb{\mu})$, later in the section, that depends only on the gaps of the expected rewards of various experts from the optimal one. 

For the Algorithm~\ref{alg:dUCB} that uses the clipped estimator, regret scales as $\mathcal{O} (\lambda(\pmb{\mu}) M^4 \log^2(6/\Delta) \log T/\Delta )$. Similarly, for the case of the median of means estimator, regret scales as $\mathcal{O} (\lambda(\pmb{\mu}) \sigma^4 \log T/\Delta )$. Here $M$ is the maximum log-divergence and $\sigma^2$ is the maximum chi-square divergence between two experts, respectively.

When the gaps between the optimum expert and sub-optimal ones are distributed uniformly at random in $[\delta_2, 1]$ $(\delta_2 > 0)$, we show that the $\lambda(\pmb{\mu})$ parameter is at most $O(\log N)$ in expectation. In contrast, if the experts were used as separate arms, a naive application of UCB-1~\cite{auer2002using} bounds would yield a regret scaling of $\mathcal{O} \left(\frac{N}{\Delta} \log T \right)$. This can be prohibitively large when the number of experts are large.  

%
%
For ease of exposition of our results, let us re-index the experts using indices $\{(1),(2),...,(N) \}$ such that $0 = \Delta_{(1)} \leq \Delta_{(2)} \leq ... \leq \Delta_{(N)}$. The regret guarantees for our clipped estimator are provided under the following assumption. 

\begin{assumption}
	\label{assump1}
	Assume the log-divergence terms~\eqref{def:mij} are bounded for all $i,j \in [N]$. Let $M = \max_{i,j} M_{ij}$. 
\end{assumption}

Now we are at a position to present one of our main theorems that provides regret guarantees for Algorithm~\ref{alg:dUCB} using the estimator~\eqref{eq:est1}. 

\begin{theorem}
	\label{thm:r1}
	Suppose Assumption~\ref{assump1} holds. Then the regret of Algorithm~\ref{alg:dUCB} at time $T$ using estimator~\eqref{eq:est1},  is bounded as follows:
	\begin{align*}
	&R(T) \leq \frac{C_1M^4 \log^2 (96M/(\gamma\Delta _{(N)})) \log T}{\gamma^2\Delta_{(N)}} + \frac{\pi^2}{3} \left(\sum_{i = 2}^{N} \Delta_{(i)} \right) +  \sum_{k = 2}^{N-1} \frac{C_1M^4 \log^2 (96M/(\gamma\Delta _{(k)})) \log T}{\gamma^2\Delta_{(k))}} \left(1 - \frac{\alpha (\Delta_{(k)})}{\alpha(\Delta_{(k+1)})} \right) 
	\end{align*}
when $T$ is such that $\frac{9C^2 \log T \log^2\left( \nicefrac{6C}{\Delta_k}  \right)}{\Delta_k^2} \geq T_1$ for all $k \neq k^*$ and $T_1 := \min\left\{ t: \beta'(t) := Cw\left(\frac{\sqrt{c_1t\log t}}{t} \right)\leq \gamma\right\}$. 
	Here, $C_1$ is an universal constant and $\alpha(x) = \frac{x^2}{\log ^2 (96M/(x\gamma))}$.
\end{theorem}

We defer the proof of Theorem~\ref{thm:r1} to Appendix~\ref{sec:clipped}. We now present Theorem~\ref{thm:r2} that provides regret guarantees for Algorithm~\ref{alg:dUCB} using the estimator~\eqref{eq:est2}. The theorem holds under the following assumption. 

\begin{assumption}
	\label{assump2}
	Assume the chi-square terms~\eqref{def:sij} are bounded for all $i,j \in [N]$. Let $\sigma = \max_{i,j} \sigma_{ij}$. 
\end{assumption}

\begin{theorem}
	\label{thm:r2}
	Suppose Assumption~\ref{assump2} holds. Then the regret of Algorithm~\ref{alg:dUCB} at time $T$ using estimator~\eqref{eq:est2},  is bounded as follows:
	\begin{align*}
	&R(T) \leq \frac{C_2\sigma^4 \log T}{\Delta_{(N)}} +  \sum_{k = 2}^{N-1} \frac{C_2\sigma^4 \log T}{\Delta_{(k)}} \left(1 - \frac{\Delta_{(k)}^2}{\Delta_{(k+1)}^2} \right) + \frac{\pi^2}{3} \left(\sum_{i = 2}^{N} \Delta_{(i)} \right) 
	\end{align*}
	Here, $C_2$ is an universal constant. 
\end{theorem}

The proof of Theorem~\ref{thm:r2} has been deferred to Appendix~\ref{sec:mom}. Now, we will delve deeper into the instance dependent terms in Theorems~\ref{thm:r1} and~\ref{thm:r2}. The proofs of Theorem~\ref{thm:r1} and \ref{thm:r2} imply the following corollary.

\begin{corollary}
	\label{cor:simple}
	Let $\lambda(\pmb{\mu}) \triangleq 1 + \sum_{k = 2}^{N-1}  \left(1 - \frac{\Delta_{(k)}^2}{\Delta_{(k+1)}^2} \right) $. We have the following regret bounds: 
	
	$(i)$ For Algorithm~\ref{alg:dUCB} with estimator~\eqref{eq:est1}, $R(T) \leq  \mathcal{O} \left(\frac{M^4 \log^2 (96M/(\gamma\Delta _{(2)})) \log T}{\Delta_{(2)}} \min \left(\lambda(\pmb{\mu}) , \frac{1}{\Delta _{(2)}} \right) \right)$ if for all $k$, $\frac{\alpha(\Delta_k)}{\alpha(\Delta_{k+1})} \geq \frac{\Delta^2_k}{\Delta^2_{k+1}}$ and the condition on $T$ from Theorem~\ref{thm:r1} holds. 
	
	$(ii)$ Similarly for Algorithm~\ref{alg:dUCB} with estimator~\eqref{eq:est2}, $R(T) \leq  \mathcal{O} \left(\frac{ \sigma^4  \log T}{\Delta _{(2)}} \min \left(\lambda(\pmb{\mu} ) , \frac{1}{\Delta _{(2)}} \right) \right)$.
\end{corollary} 

Corollary~\ref{cor:simple} leads us to our next result. In Corollary~\ref{lem:mu} we show that when the $\Delta$ gaps are uniformly distributed, then the $\lambda(\pmb{\mu})$ is $O(\log N)$, in expectation.

\begin{corollary}
	\label{lem:mu}
	Consider a generative model where $\Delta_{(3)}\leq...\leq\Delta_{(N)}$ are the order statistics of $N-2$ random variables drawn i.i.d uniform over the interval $[\Delta_{(2)},1]$. Let $p_{\Delta}$ denote the measure over these $\Delta$'s. Then we have the following:
	
	($i$) For Algorithm~\ref{alg:dUCB} with estimator~\eqref{eq:est1}, $\EE_{p_{\Delta}} \left[ R(T)\right] = \mathcal{O} \left(\frac{M^4 \log N \log^2(1/\Delta_{(2)}) \log T }{\Delta_{(2)}} \right)$.
	
	($ii$) For Algorithm~\ref{alg:dUCB} with estimator~\eqref{eq:est2}, $\EE_{p_{\Delta}} \left[ R(T)\right] = \mathcal{O} \left(\frac{\sigma^4 \log N \log T }{\Delta_{(2)}} \right)$.
\end{corollary}

%
%
%

\begin{remark}
Note that our guarantees do not have any term containing $K$ - the number of arms. This dependence is implicitly captured in the divergence terms among the experts. In fact when the number of arms $K$ is very large, we expect our divergence based algorithms to perform comparatively better than other algorithms, whose guarantees explicitly depend on $K$. This phenomenon is observed in practice in our empirical validation on real world data-sets in Section~\ref{sec:sims}. We also show empirically, that the term $\lambda(\pmb{\mu})$ grows very slowly with the number of experts on real-world data-sets. This empirical result is included in Appendix~\ref{sec:more_emp}. 
\end{remark}

\section{Empirical Results}
\label{sec:sims}

\begin{figure*}
	\centering
	
	\subfloat[][]{\includegraphics[width = 0.29\linewidth]{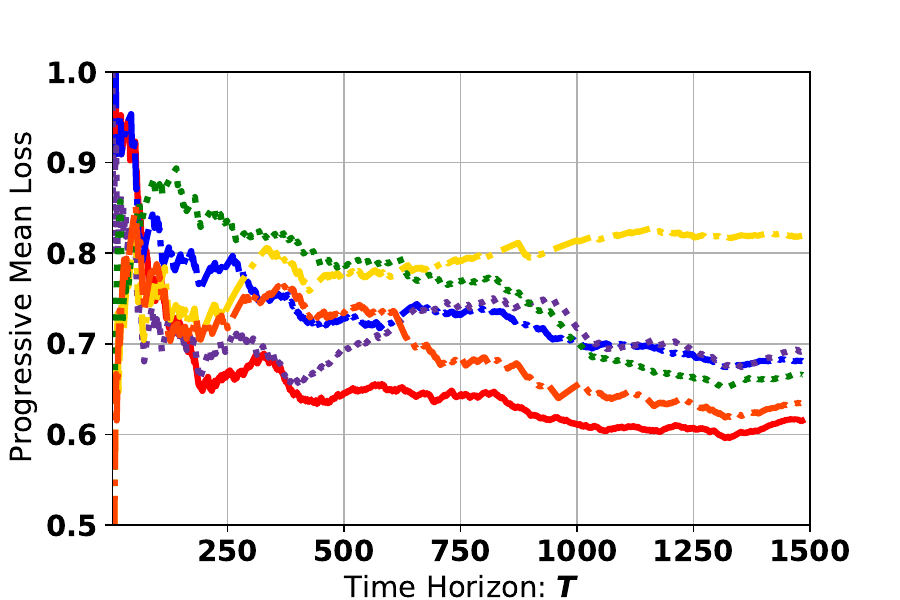}\label{fig:drug}} \hfill
	\subfloat[][]{\includegraphics[width = 0.29\linewidth]{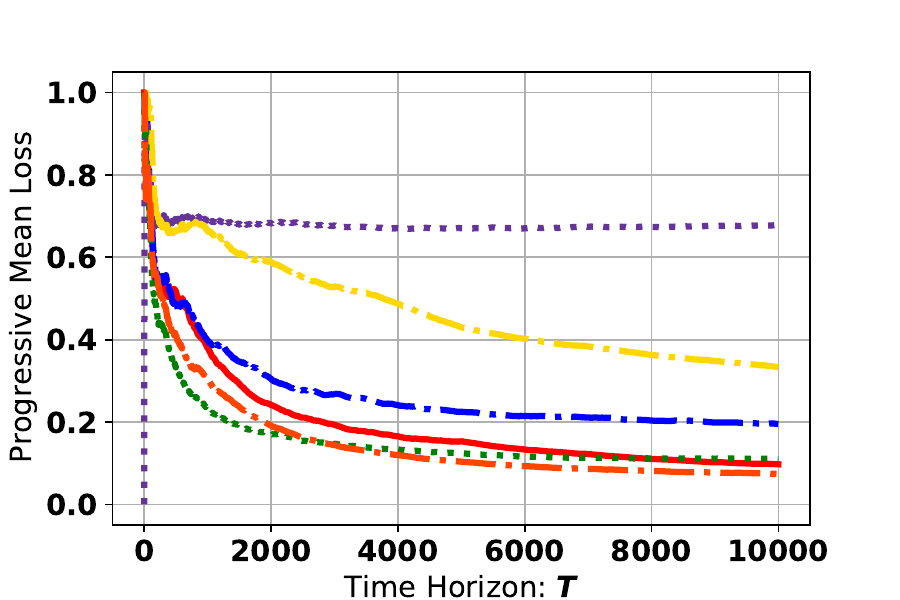}\label{fig:stream}} \hfill
	\subfloat[][]{\includegraphics[width = 0.29\linewidth]{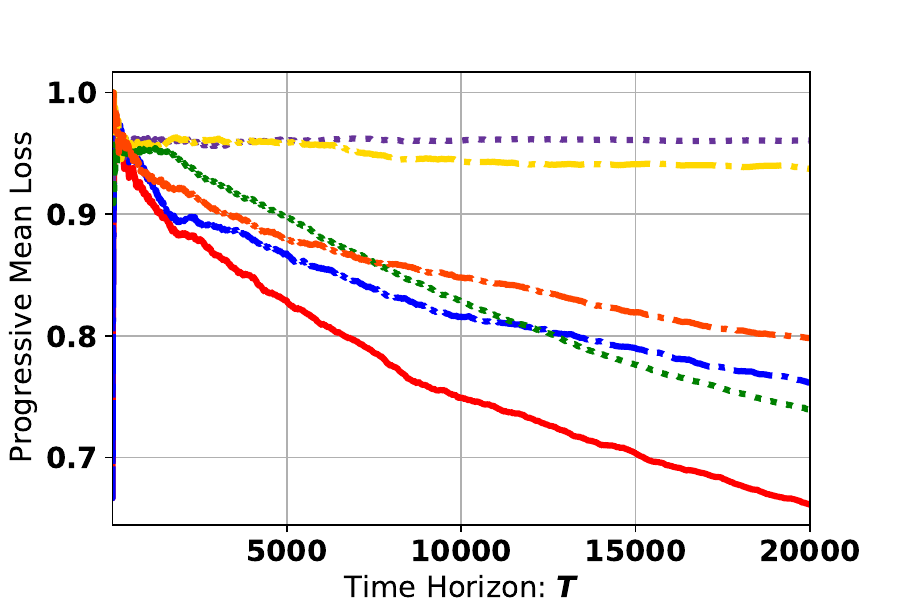}\label{fig:letter}} \hfill 
	\subfloat[][]{\includegraphics[width = 0.10\linewidth]{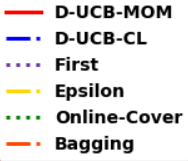}\label{fig:legend}}
	\caption{ \small In all these plots, the progressive mean loss~\cite{agarwal2014taming} till time $T$ has been plotted as a function of time $T$. $(a)$ Performance of the algorithms on the Yeast dataset~\cite{horton1996probabilistic}. $(b)$ Performance of the algorithms on the Stream Analytics dataset~\cite{stream}. $(c)$ Performance of the algorithms on the Letters dataset~\cite{frey1991letter}. $(d)$ Legend.}
	\label{fig:sims_regret}. 
\end{figure*}

In this section, we will empirically test our algorithms on three real-world multi-class classification datasets, against other state of the art algorithms for contextual bandits with experts. Any multi-class classification dataset can be converted into a contextual bandit scenario, where the features are the contexts. At each time-step, the feature (context) of a sample point is revealed, following which the contextual bandit algorithm chooses one of the $K$ classes, and the reward observed is $1$ if its the correct class otherwise it is $0$. This is bandit feedback as the correct class is never revealed, if not chosen. This method has been widely used to benchmark contextual bandit algorithms~\cite{beygelzimer2011contextual,agarwal2014taming}, and is in fact implemented in Vowpal Wabbit~\cite{wabbit}.

Our algorithm is run in batches. At the starting of each batch, we add experts trained on prior data through cost-sensitive classification oracles~\cite{beygelzimer2009offset} and also update the divergence terms between experts, which are estimated from data observed \textit{so far}. During each batch, Algorithm~\ref{alg:dUCB} is deployed with the current set of experts. The pseudo-code for this procedure is provided in Algorithm~\ref{alg:batched}. 
\begin{algorithm}
	\caption{Batched D-UCB with cost-sensitive classification experts}
	\begin{algorithmic}[1]
		\State Let $\Pi = \{\pi_1\}$, which is an expert that chooses arms randomly. For time steps $t = 1$ to $3K$, choose an arm sampled from expert $\pi_1$. $t = 3K+1$.  
		\State Add experts to $\Pi$ trained on observed data and update divergences. 
		\While {$t <= T$}
		\For {$s = t$ to $t + \mathcal{O}(\sqrt{t})$}
		\State Deploy Algorithm~\ref{alg:dUCB} with experts in $\Pi$.
		\EndFor
		\State Let $t = t + \mathcal{O}(\sqrt{t})$. Add experts to $\Pi$ trained on observed data and update divergences.
		\EndWhile
	\end{algorithmic}
	\label{alg:batched}
\end{algorithm}
We use XgBoost~\cite{chen2016xgboost} and Logistic Regression in scikit-learn~\cite{buitinck2013api} with calibration, as the base classifiers for our cost-sensitive oracles. Bootstrapping is used to generate different experts. At the starting of each batch $4$ new experts are added. The constants are set as $C=1, c_1 = 1, c_2 = 4$ and $c_3 = 2$ in practice. All the settings are held fixed over all three data-sets, {\it without any parameter tuning}. We provide more details in Appendix~\ref{sec:more_emp}. In the appendix we also show that the gap dependent term in our theoretical bounds grows much slower compared to UCB-1 bounds (Fig.~\ref{fig:terms}), as the number of experts increase in the stream analytics dataset~\cite{stream}. An implementation of our algorithm can be found \href{https://github.com/rajatsen91/CB_StochasticExperts}{here} \footnote{https://github.com/rajatsen91/CB\_StochasticExperts}.

We compare against Vopal Wabbit implementations of the following algorithms: $(i)$ $\epsilon$-greedy~\cite{langford2008epoch} - parameter set at '--epsilon 0.06'. $(ii)$ First (Greedily selects best expert) - parameter set at '--first 100'. $(iii)$ Online Cover~\cite{agarwal2014taming} - parameter set at '--cover 5' $(iv)$ Bagging (Simulates Thompson Sampling through bagged classifiers) - parameter set at '--bag 7'.  

{\bf Yeast Data: } This dataset~\cite{horton1996probabilistic} is a part of UCI repository. It has data from $1484$ instances with $8$ dimensional continuous features (contexts). There are $10$ different localization sites for protiens that can be used as classes or labels ($10$ arms). The performance of the algorithms are shown in Fig.~\ref{fig:drug}. We see that D-UCB (Algorithm~\ref{alg:batched}) with median of moments clearly performs the best in terms of average loss, followed by the bagging approach. D-UCB with median of moments converges to an average loss of $0.61$ while that of  bagging is $0.63$. 

{\bf Stream Analytics Data: } This dataset~\cite{stream} has been collected using the stream analytics client. It has $10000$ samples with $100$ dimensional mixed features (contexts). There are $10$ classes ($10$ arms). For each entry, if the bandit algorithm selects the correct class, the reward observed is $1$, o.w. $0$ reward is observed. The performance of the algorithms are shown in Fig.~\ref{fig:stream}. In this data-set bagging performs the best closely followed by D-UCB-MOM (Algorithm~\ref{alg:batched}) and Online-Cover. Bagging is a strong competitor empirically, however this algorithm lacks theoretical guarantees. Bagging converges to an average loss of $8\%$, while D-UCB with median of moments converges to an average loss of $9\%$.

{\bf Letters Data: } This dataset~\cite{frey1991letter} is a part of the UCI repository. It has $20000$ samples of hand-written English letters, each with $17$ hand-crafted visual features (contexts). There are $26$ classes ($26$ arms) corresponding to $26$ letters. For each entry, if the bandit algorithm selects the correct letter, the reward observed is $1$, o.w. $0$ reward is observed. The performance of the algorithms are shown in Fig.~\ref{fig:letter}. D-UCB-MOM significantly outperform the other algorithms.

\subsection*{Conclusion}
We study the problem of contextual bandits with stochastic experts. We propose two UCB style algorithms, that use two different importance sampling estimators, which can leverage \textit{information leakage} between the stochastic experts. We provide instance-dependent regret guarantees for our UCB based algorithms. Our algorithms show strong empirical performance on real-world datasets. We believe that this paper introduces an interesting problem setting for studying contextual bandits, and opens up opportunities for future research that may include better regret bounds for the problem and an instance-dependent lower-bound. 

\subsubsection*{Acknowledgment} 
This work is partially supported by NSF SaTC 1704778, ARO W911NF-17-1-0359, and the US DoT supported D-STOP Tier 1 University Transportation Center.

\FloatBarrier
\FloatBarrier
\bibliographystyle{plain}
\bibliography{experts.bib}

\clearpage
\appendix

\section{Clipped Estimator}
\label{sec:clipped}

As mentioned in Section~\ref{sec:estimators}, the motivating equation guiding the design of our estimators is Eq.~\eqref{eq:infoleakage}. This equation tells us that even when the statistics of the samples observed are governed by the distribution of $(X,V,Y)$ under expert $j$, we can infer the mean of expert $k$. Such observations were made in~\cite{lattimore2016causal,pmlr-v70-sen17a} in the context of best arm identification problems. Suppose we observe $t$ samples under expert $\pi_j$. Guided by Eq.~\eqref{eq:infoleakage}, one might come up with the following naive importance sampled estimator for the mean under expert $k$ ($\mu_k$):
\begin{align*}
\hat{\mu}_k = \frac{1}{t}\sum_{s = 1}^{t} Y_j\frac{\pi_k(V_j(s) \vert
	X_j(s))}{\pi_j(V_j(s) \vert X_j(s))}. 
\end{align*}
However, it is not possible to derive good confidence interval for the above estimator because even though the reward variable $Y$ is bounded, the reweighing term $\pi_k(V_j(s) \vert
	X_j(s))/\pi_j(V_j(s) \vert X_j(s))$ can be unbounded and in some case heavy-tailed. The key idea is to come up with robust estimators that have good variance properties. One approach of controlling the variance of such estimators is to clip that the samples that are too large. This leads to the following clipped estimator~\cite{pmlr-v70-sen17a}:

\begin{align*}
\label{eq:clippedtwo}
\hat{\mu}_k = &\frac{1}{t}\sum_{s = 1}^{t} Y_j(s)\frac{\pi_k(V_j(s) \vert
	X_j(s))}{\pi_j(V_j(s) \vert X_j(s))} \times \mathds{1} \left(\frac{\pi_k(V_j(s) \vert
	X_j(s))}{\pi_j(V_j(s) \vert X_j(s))} \leq \eta_{kj} \right). \numberthis
\end{align*}

The clipping makes the estimator biased, however it helps in controlling the variance. The clipper level $\eta_{kj}$ which depends on the relationship between $\pi_k$ and $\pi_j$ needs to be set carefully to control the bias-variance trade-off. In~\cite{pmlr-v70-sen17a}, it has been shown that if the log-divergence measure $M_{kj}$ (defined in~\eqref{def:mij}) is bounded, then a good choice is $\eta_{kj} = 2 \log(2/\epsilon) M_{kj}$, if we want an additive bias of at most $\epsilon (t)/2$ (Theorem 3 in~\cite{pmlr-v70-sen17a}).  

This idea can be generalized to estimating the mean of expert $k$, while observing samples from all the other experts. This leads to the clipped estimator in Eq.~\eqref{eq:est1}. In what follows, we will abbreviate $\EE_{p_j(.)}[.]$ as $\EE_j[.]$. In this section let $\hat{\mu}_k(t) = \hat{\mu}^c_k(t)$. 

Recall that the for all experts $i\in [N]$, we have that $\mu_i \geq \gamma$. Specifically, this implies that for the ``worst'' expert, we have that $\min_{i\in N}\mu_i \geq \gamma.$ The following lemma establishes concentrations for our clipped estimator in Eq.~\eqref{eq:est1}.

\begin{lemma}
	\label{lem:clipped}
For any expert $j\in [N]$, the estimator $\hat\mu_j(t)$ defined in Eq.\eqref{eq:est1} satisfies
\begin{align*}
\PP\left( (1 - \chi(t))\left(\mu_j - \frac{\epsilon(t)}{2}\right) \leq \hat{\mu}_j(t) \leq (1 + \chi(t)) \mu_j\right ) \geq 1 -  2\exp \left( - \frac{\gamma^2 \chi(t)^2t}{128 M^2 (\log (2 / \epsilon(t)))^2 }\right).
\end{align*}
when $\chi(t)$ and $\epsilon(t) < \gamma$ are fixed non-negative constants. 
\end{lemma}

\begin{proof}
Fix any expert $j\in [N]$. Let $k(l)$ denote the expert that was chosen at time-slot $l$. Define the following martingale: let $\cF_s$ denote the filtration that is formed by the observations until time $s$ \emph{and} the expert chosen at time $s+1$. Note that given that $k(l)$ is fixed, $M_{jk(l)}$ is a fixed constant which denotes the divergence with between experts $j$ and $k(l)$. With $A_0 = 0$, define

\begin{align*}
A_s = \sum_{l=1}^{s} \frac{L_j(l)}{M_{jk(l)}} -  \sum_{l=1}^{s} \EE\left[ \frac{L_j(l)}{M_{jk(l)}} \bigg \vert \mathcal{F}_{l-1} \right]. 
\end{align*}

where to ease notation, we define 
\begin{align*}
r_l &:= \frac{\pi_j(V_{k(l)}(l) \vert X_{k(l)}(l))}{\pi_{k(l)}(V_{k(l)}(l) \vert X_{k(l)}(l))}\\
\alpha_l &:= 2\log(2/\epsilon(t))M_{jk(l)},\\
L_j(l) &:=  Y_l \times r_l \times\mathds{1}\left\{ r_l \leq \alpha_l \right\} .
\end{align*}

As $\cF_{l-1}$ contains the expert chosen at time $l$, we can write with $\mu_j(l) := \EE[L_j(l) | F_{l-1}]$ and $B_t = \sum_{l = 1}^t \frac{L_j(l)}{M_{jk(l)}}$,
\begin{align*}
A_s = B_s -  \sum_{l=1}^{s} \frac{\mu_j(l)}{M_{jk(l)}},
\end{align*}
Note that $|A_s - A_{s-1}| \leq 4\log(2/\epsilon(t))$. Therefore we can apply the Azuma-Hoeffding inequality for martingales to write 
\begin{align}
\PP \left( \bigg \lvert B_t -  \sum_{l=1}^{t} \frac{\mu_j(l)}{M_{jk(l)}} \bigg \rvert \geq \chi\right) \leq 2\exp \left( - \frac{\chi^2}{32 t (\log (2 / \epsilon(t)))^2 }\right). \label{eq:martingale}
\end{align}

Now we consider each of the tails separately. 

\noindent\textbf{1.} We have that 
\begin{align*}
\PP \left(  B_t \geq \sum_{l=1}^{t} \frac{\mu_j(l)}{M_{jk(l)}} + \chi \right) \leq \exp \left( - \frac{\chi^2}{32t (\log (2 / \epsilon(t)))^2 }\right)
\end{align*}
Consider the following chain:
\begin{align*}
(1+\chi(t))\left( \max_{l\in[t]} \mu_j(l)\right) \sum_{l = 1}^t \frac{1}{M_{jk(l)}} &\geq \sum_{l = 1}^t \frac{\mu_j(l)}{M_{jk(l)}}  + \frac{\chi(t) t}{M}\max_{l\in [t]} \mu_j(l)\\
&\geq \sum_{l = 1}^t \frac{\mu_j(l)}{M_{jk(l)}}  + \frac{\gamma\chi(t) t}{2M}
\end{align*}

where the final inequality comes about as a consequence of $\epsilon(t)<\gamma$ and $\mu_j(l) \geq \mu_j -\frac{\epsilon(t)}{2}\geq \gamma - \frac{\epsilon(t)}{2}$. The latter is proved in Lemma \ref{lem:mean}. We also use the fact that $M := \max_{j,k} M_{jk}$. Thus, we have 
\begin{align*}
&  \PP \left(  B_t \geq (1 + \chi(t) ) \left(\max_l {\mu_j(l)} \right)\sum_{l=1}^{t} \frac{1}{M_{jk(l)}}  \right) \stackrel{(i)}{\leq} \PP \left(  B_t \geq \sum_{l=1}^{t} \frac{\mu_j(l)}{M_{jk(l)}} + \frac{\gamma\chi(t) t}{2M} \right) \\
& \leq \exp \left( - \frac{\gamma^2\chi(t)^2t}{128 M^2 (\log (2 / \epsilon(t)))^2 }\right). 
\end{align*}
This implies that,
\begin{align}
\PP \left( \hat{\mu}_k(t) \geq (1 + \chi(t) ) (\max_l {\mu_j(l)} ) \right) \leq \exp \left( - \frac{\gamma^2\chi(t)^2t}{128 M^2 (\log (2 / \epsilon(t)))^2 }\right). \label{eq:upper_tail}
\end{align}
	
\noindent\textbf{2.} We have that 
\begin{align*}
\PP \left(  B_t \leq \sum_{l=1}^{t} \frac{\mu_j(l)}{M_{jk(l)}} - \chi \right) \leq \exp \left( - \frac{\chi^2}{32t (\log (2 / \epsilon(t)))^2 }\right)
\end{align*}

Using similar arguments as above, we can write

\begin{align*}
&  \PP \left(  B_t \leq (1 - \chi(t) ) (\min_l {\mu_j(l)} )\sum_{l=1}^{t} \frac{1}{M_{jk(l)}}  \right) \leq \PP \left(  B_t \leq \sum_{l=1}^{t} \frac{\mu_j(l)}{M_{jk(l)}} - \frac{\gamma\chi(t) t}{2M} \right) \\
& \leq \exp \left( - \frac{\gamma^2\chi(t)^2t}{128 M^2 (\log (2 / \epsilon(t)))^2 }\right). 
\end{align*}

This implies that,
\begin{align}
\PP \left( \hat{\mu}_k(t) \leq (1 - \chi(t) ) (\min_l {\mu_j(l)} ) \right) \leq \exp \left( - \frac{\gamma^2\chi(t)^2t}{128 M^2 (\log (2 / \epsilon(t)))^2 }\right). \label{eq:lower_tail}
\end{align}
Since the choice of $j\in[N]$ was arbitrary, Combining eq.~\ref{eq:upper_tail}, ~\ref{eq:lower_tail} and \ref{eq:prev_paper} we have the lemma. 
\end{proof}

Now, we prove the bias results for $\mu_j(l)$ we use above.
\begin{lemma}
	\label{lem:mean}
	For all times $l$ and all experts $j\in [N]$,  we have,
	\begin{align}
	\mu_j - \frac{\epsilon(t)}{2} \leq \mu_j(l) \leq \mu_j \label{eq:prev_paper}.
	\end{align}
	where $\mu_j(l)$ is defined in the proof of Lemma~\ref{lem:clipped}.
\end{lemma}

\begin{proof}
We first note that under the filtration $\cF_{l-1}$, $M_{jk(l)}$ is a constant and $k(l)$ is fixed. Therefore, following the notation in Lemma \ref{lem:clipped},  we have the following chain,
\begin{align*}
& \mu_j(l) = \EE \left[L_j(l) \vert \cF_{l-1} \right] \\
&= \EE_{k(l)} \left[ Y_l\frac{\pi_j(V_{k(l)}(l) \vert X_{k(l)}(l))}{\pi_{k(l)}(V_{k(l)}(l) \vert X_{k(l)}(l))}\right] - \EE_{k(l)} \left[Y_l \frac{\pi_j(V_{k(l)}(l) \vert X_{k(l)}(l))}{\pi_{k(l)}(V_{k(l)}(l) \vert X_{k(l)}(l))}\times\mathds{1}\left\{ r_l > \alpha_l\right\} \right] \\
&\stackrel{(i)}{\geq} \mu_j - \PP_{j}\left( \frac{\pi_j(V_{k(l)}(l) \vert X_{k(l)}(l))}{\pi_{k(l)}(V_{k(l)}(l) \vert X_{k(l)}(l))}  > 2\log(2/\epsilon (t))M_{jk(l)} \right) \\
& \stackrel{(ii)}{\geq} \mu_j - \frac{\epsilon(t)}{2}. 
\end{align*}

Here, (i) follows from the fact that $Y \in [0,1]$ and (ii) follows from Lemma 2 in~\cite{pmlr-v70-sen17a}. 
\end{proof}

Now, we proceed by bounding the probabilities of bad events for the best expert and suboptimal experts separately. We recall that the index of an expert $k\in [N]$ is set to be $U_k(t) = \hat\mu_k(t) + \frac{3}{2}\beta(t)$ where $\beta(t) = Cw\left(\frac{\sqrt{c_1 t \log t}}{Z_k(t)} \right)$. Here, C = $\frac{16M}{\gamma}, c_1 = 1, w(x) = y \iff \frac{y}{\log(2/y)} = x$. 

For the best arm, we establish the following confidence result.

\begin{lemma}
	\label{lem:ubound} 
	Define $T_1 = \min\left\{ t: \beta'(t) := Cw\left(\frac{\sqrt{c_1t\log t}}{t} \right)\leq \gamma\right\}$. Then, for all $t\geq T_1$, the index of the best arm formed using the Clipped Estimator satisfies
	\begin{align*}
	\PP \left( U_{k^*}(t) > \mu^* \right) \geq 1 - \frac{1}{t^2},
	\end{align*}
\end{lemma}

\begin{proof}
	Since $Z_k(t)\leq t$ and $w(x)$ is increasing, we have that $\beta(t)\geq \beta'(t)$. Now, we have the following chain,
	\begin{align*}
	\PP \left( U_{k^*}(t) \leq \mu^* \right) &= \PP \left( \hat{\mu}_{k^*}(t) \leq \mu^* - \frac{3}{2}C w \left(\frac{\sqrt{c_1 t \log t}}{Z_k(t)}\right) \right) \\
	& \stackrel{(i)}{\leq} \PP \left( \hat{\mu}_{k^*}(t) \leq \mu^* -\mu^*C w \left(\frac{\sqrt{c_1 t \log t}}{t}\right) - \frac{1}{2}C w \left(\frac{ \sqrt{c_1 t \log t}}{t}\right) \right) \\
	&\stackrel{(ii)}{\leq} \PP \left( \hat{\mu}_{k^*}(t) \leq \mu^* -\mu^* \beta'(t) - (1 - \beta'(t)) \frac{1}{2} \beta'(t) \right) \\
	&\leq \PP\left(\hat\mu_{k^*}(t) \leq (1-\beta'(t))\left(\mu^* - \frac{\beta'(t)}{2}\right)  \right)
	\end{align*}
	
Here, (i) uses the observation above and that $\mu^* \leq 1$. (ii) follows from the fact that $\beta'(t) \leq 1$. Now, we use Lemma \ref{lem:clipped} with $\chi(t)$ set to be sample path independent $\beta'(t)$ to write 
\begin{align*}
\PP \left( U_{k^*}(t) \leq \mu^* \right) &\leq \exp\left( -\frac{\gamma^2 \beta'(t)^2 t}{128M^2 (\log(2/\beta'(t)))^2} \right)
\end{align*}

Consider the exponent
\begin{align*}
\frac{\gamma^2t}{128M^2} \left( \frac{\beta'(t)}{\log(2/\beta'(t))}\right)^2 &= \frac{\gamma^2C^2t}{128M^2}\left( \frac{w\left(\sqrt{\nicefrac{\log t}{t}}\right)}{\log\left(\nicefrac{2}{Cw\left(\sqrt{\nicefrac{\log t}{t}}\right)}\right)} \right)^2\\
&= 2t \left( \frac{w\left(\sqrt{\nicefrac{\log t}{t}}\right)}{\log\left(\nicefrac{2}{Cw\left(\sqrt{\nicefrac{\log t}{t}}\right)}\right)} \right)^2\\
&\geq 2t\left(\sqrt{\frac{\log t}{t}}\right)^2
\end{align*}
Here, the inequality holds due to the following chain of reasoning:
\begin{align*}
\frac{w(x)}{\log\left(\nicefrac{2}{w(x)}\right)} &= x\\
\frac{w(x)}{\log\left(\nicefrac{2}{aw(x)}  \right)} &= x \left(  \frac{\log (\nicefrac{2}{w(x)})}{\log( \nicefrac{2}{aw(x)} )} \right)
\end{align*}
The RHS here is greater than  or equal to $x$ if $a\geq 1$, which is the case with $a = C$ since $M\geq 1, \gamma \leq 1$. Using this lower bound for the exponent in the original expression gives us the required result.

\end{proof}

Next we prove that for a large enough time $t$, the UCB estimate of the $k^{th}$ expert is less than that of $\mu^*$.

\begin{lemma}
	\label{lem:lbound}
	
	Let $T_1$ be as defined in Lemma~\ref{lem:ubound}. Then, for all $t \geq T_k := \max\left\{T_1,\frac{9C^2 M^2 \log T \log^2\left( \nicefrac{6C}{\Delta_k}  \right)}{\Delta_k^2} \right\}$, we have that for any sub-optimal arm $k\neq k^*$,
	\begin{align*}
	\PP \left( U_{k}(t) < \mu^* \right) \geq 1 - \frac{1}{t^2}.
	\end{align*}
\end{lemma}

\begin{proof}
	We have that $Z_k(t)\geq \frac{t}{M}$ and $t\geq\frac{9C^2 M^2 \log T \log^2\left( \nicefrac{6C}{\Delta_k}  \right)}{\Delta_k^2}.$ Additionally, $w(x)$ is increasing in $x$. Therefore, we can write that
	\begin{align*}
	\frac{3C}{2}w \left( \frac{\sqrt{c_1 t \log t}}{Z_k(t)}\right) &\leq \frac{3C}{2}w \left( \frac{M\sqrt{c_1 t \log t}}{t}\right)\\
	&\leq \frac{3C}{2}w \left( \frac{M\sqrt{c_1 t \log T}}{t}\right) \nonumber \\
	&\leq \frac{3C}{2}w \left( \frac{M\sqrt{c_1  \log T}}{\frac{3CM\sqrt{c_1\log T} \log \left(\nicefrac{6C}{\Delta_k} \right)}{\Delta_k}}\right)\\
	&= \frac{3C}{2} w\left( \frac{\nicefrac{\Delta_k}{3C} }{\log\left(  \frac{2}{\nicefrac{\Delta_k}{3C}}  \right)} \right)\\
	&= \frac{\Delta_k}{2} 
	\end{align*}
	Where the last equality follows from the fact that $w(x) = y \iff \frac{y}{\log\left(\nicefrac{2}{y}\right)} = x$. Using the argument above along with the fact that $\Delta_k = \mu^* - \mu_k$ and $\mu_k<\mu^*\leq 1$, we have the following chain:
	
	\begin{align*}\label{eq:lboundeq}
	\PP \left( U_{k}(t) > \mu^* \right) &= \PP \left(\hat{\mu}_k(t) > \mu^* - \frac{3 \beta(t)}{2} \right) \\
	&\leq \PP \left(\hat{\mu}_k(t) > \mu^* - \frac{\Delta_k}{2}\right) \\
	&\leq  \PP\left(\hat\mu_k(t) > \mu_k + \frac{\Delta_k}{2} \right)\\
	&\leq \PP \left(\hat{\mu}_k(t) > \mu_k\left( 1 + \frac{\Delta_k}{2}\right) \right) \\
	& \stackrel{(i)}{\leq} \PP \left(\hat{\mu}_k(t) > \mu_k\left( 1 + \frac{3C}{2}w \left( \frac{M \sqrt{t \log t}}{t}\right) \right) \right) \\
	&\leq \exp\left(-\frac{\gamma^2 \left(\nicefrac{3C}{2}\right)^2 t }{128M^2} \times \left( \frac{w\left(M\sqrt{\nicefrac{\log t}{t}}\right)}{\log \left(\nicefrac{2}{Cw\left(\sqrt{\nicefrac{\log t}{t}} \right) } \right)} \right)^2\right)\numberthis
	\end{align*}
	
In (i) we have used the fact that $\Delta_k/2 \geq \frac{3C}{2}w \left( \frac{M \sqrt{t \log t}}{t}\right)$ from the chain just before (note that $c_1=1$). Here, the final inequality applies Lemma~\ref{lem:clipped} (bounds for the upper tail error) with $\chi(t) = \frac{3C}{2}w \left( \frac{M\sqrt{ t \log t}}{t}\right)$ and $\epsilon(t) = \beta'(t)$ defined in Lemma~\ref{lem:ubound}.  

Considering the exponent alone, we have that
\begin{align*}
 \frac{\gamma^2 \left(\nicefrac{3C}{2}\right)^2 t }{128M^2} \times \left( \frac{w\left(M\sqrt{\nicefrac{\log t}{t}}\right)}{\log \left(\nicefrac{2}{Cw\left(\sqrt{\nicefrac{\log t}{t}} \right) } \right)} \right)^2 &= \frac{9}{2}t\times \left( \frac{  w\left(M\sqrt{\nicefrac{\log t}{t}} \right)}{\log\left(\nicefrac{2}{Cw\left(\sqrt{\nicefrac{\log t}{t} }\right)}\right) }\right)^2\\
 &\geq 2t \times \left( \frac{  w\left(\sqrt{\nicefrac{\log t}{t}} \right)}{\log\left(\nicefrac{2}{Cw\left(\sqrt{\nicefrac{\log t}{t} }\right)}\right) }\right)^2\\
 &\geq 2t\left(\sqrt{\frac{\log t}{t}} \right)^2
\end{align*}
Here, the first inequality uses that $w(x)$ is increasing and $M\geq 1$, while the final inequality follows from the reasoning in the proof of Lemma~\ref{lem:ubound}. Substituting this inequality in Equation~\eqref{eq:lboundeq} gives us the result.

\end{proof}

\begin{proof}[Proof of Theorem~\ref{thm:r1}]

	Firstly, we note that together, Lemmas~\ref{lem:ubound} and \ref{lem:lbound} imply that for $k\neq k^*$, for all $t\geq T_k$, it holds that 
	
	\begin{align}
	\label{eq:badevent}
	\PP \left(k(t) = k \right) \leq \frac{2}{t^2}
	\end{align}
	 for $T_k$ defined as in Lemma~\ref{lem:lbound}. Further, recall that the experts are indexed as $\{ (1),(2),...,(N)\}$ with $0 = \Delta_{(1)} \leq \Delta_{(2)}\leq ... \leq \Delta_{(N)}$ and thus, we have that the sequence of times $T_{(N)},T_{(N-1)},...,T_{(2)}$ is increasing.  
	 
	 Thus, we have the following decomposition for the regret of the algorithm:
	 
	 \begin{align*}
		R_T &= \sum_{t=1}^T \mathbb{E}[\Delta_{k(t)}]\\
		&= \sum_{t=1}^T \sum_{k = 2}^N \Delta_k \PP(k(t) = k)\\
		&\leq \Delta_{(N)}\left[T_{(N)} + \frac{\pi^2}{3}\right] +\Delta_{(N-1)} \left[ \left(T_{(N-1)}-T_{(N)}\right) + \frac{\pi^2}{3}\right] +...\\
		&~~~~~~~~~~~~~+ \Delta_{(N-k-1)} \left[ \left(T_{(N-k-1)}-T_{(N-k)} \right)+ \frac{\pi^2}{3}\right] +...+ \Delta_{(2)}\left[ \left(T_{(2)}-T_{(3)} \right) + \frac{\pi^2}{3}\right]\\
		&= \Delta_{(N)}T_{(N)} + \frac{\pi^2}{3}\left(\sum_{k=2}^N \Delta_{(k)}\right)+ \sum_{k=0}^{N-3} \left( T_{(N-k-1)} - T_{(N-k)} \right)\Delta_{(N-k-1)}.
	 \end{align*}
	
	Thus, for $T$ large enough (specifically, $T$ such that $T_{(N)}>T_1$), we can write
	
	\begin{align*}
		R_T &= \frac{9C^2M^2\log T \log^2(\nicefrac{6C}{\Delta_{(N)}})}{\Delta_{(N)}} + \frac{\pi^2}{3}\sum_{k=2}^N \Delta_{(k)} \\
		&~~~~~~~+ \sum_{k=0}^{N-3} \frac{9C^2M^2\log T \log^2(\nicefrac{6C}{\Delta_{(N-k-1)}})}{\Delta_{(N-k-1)}}\left(1-\frac{\log^2 (\nicefrac{6C}{\Delta_{(N-k)}})\Delta_{(N-k-1)}^2}{\log^2(\nicefrac{6C}{\Delta_{(N-k-1)}}) \Delta_{(N-k)}^2} \right)\\
		&= \frac{2304M^4\log T \log^2(\nicefrac{96M}{\gamma\Delta_{(N)}})}{\gamma^2\Delta_{(N)}} + \frac{\pi^2}{3}\sum_{k=2}^N \Delta_{(k)} \\
		&~~~~~~~~ +\sum_{k=2}^{N-1} \frac{2304M^4\log^2(\nicefrac{96M}{\gamma\Delta_{(k)})} \log T}{\gamma^2\Delta_{(k)}}\left(1-\frac{\alpha(\Delta_{(k)})}{\alpha(\Delta_{(k+1)})} \right) \numberthis\label{eq:uselater}
	\end{align*}
	where $\alpha(x) := \frac{x^2}{\log(\nicefrac{96M}{\gamma x})}$.

\end{proof}

\section{Median of Means Estimator}
\label{sec:mom}

We will prove the following lemma,

\begin{lemma}
	\label{lem:mom}
	Let $\delta(t) \in (0,1)$. Then the estimator in~\eqref{eq:est2} has the following confidence bound,
	\begin{align}
	\PP \left( \vert \hat{\mu}^m_{k}(t) - \mu_k \vert \leq \sigma \sqrt{\frac{c'_3 \log (1/\delta(t))}{t}} \right) \geq 1 - \delta(t). 
	\end{align}
\end{lemma}
where we set $c'_3 = 32$ for the analysis. 

The median of means estimator is popular for estimating statistics under heavy-tailed distribution~\cite{bubeck2013bandits,lugosi2017sub}. We shall see that the median of means based estimator in Eq.~\eqref{eq:est2} has good variance properties, when the chi-square divergence (Assumption~\ref{assump2}) are bounded. Before proving Lemma~\ref{lem:mom}, we will be establishing some intermediate results. 

\begin{lemma}
	\label{lem:chebby}
	Consider the quantity $\hat{\mu}^{r}_j(t)$ in Eq.~\eqref{eq:momm}. The variance of this quantity is upper bounded as follows:
	\begin{align*}
	\EE \left[ \left(\hat{\mu}^{r}_j(t) - \mu_j \right)^2\right]  \leq  \frac{\sigma^2}{m}
	\end{align*}
	where $m = \floor{t/l(t)}$. 
\end{lemma}

\begin{proof}
	Let us consider the group of samples in the $r$-th group and re-index them from $i=1$ to $m$ in the order in which they were collected. Therefore, we can write the estimator as follows,
	\begin{align*}
	\hat{\mu}^{r}_j(t)=  \frac{1}{W_j(r, t)}\sum_{l=1}^{m} Y_l\frac{1}{\sigma_{jk(l)}} \frac{\pi_j(V_{k(l)}(l) \vert X_{k(l)}(l))}{\pi_{k(l)}(V_{k(l)}(l) \vert X_{k(l)}(l))}.
	\end{align*}

Recall that $W_j(r,t) = \sum_{l=1}^{m} \frac{1}{\sigma_{jk(l)}}$. 

We have the following chain,

\begin{align*}
\EE \left[ \left(\hat{\mu}^{r}_j(t) - \mu_j \right)^2\right] = \EE \left[ \left( \frac{1}{W_j(r, t)}\sum_{l=1}^{m} \left( Y_l\frac{1}{\sigma_{jk(l)}} \frac{\pi_j(V_{k(l)}(l) \vert X_{k(l)}(l))}{\pi_{k(l)}(V_{k(l)}(l) \vert X_{k(l)}(l))} - \frac{\mu_j}{\sigma_{jk(l)}} \right)\right)^2\right]
\end{align*}
 
 Let $A_l = \left( Y_l\frac{1}{\sigma_{jk(l)}} \frac{\pi_j(V_{k(l)}(l) \vert X_{k(l)}(l))}{\pi_{k(l)}(V_{k(l)}(l) \vert X_{k(l)}(l))} - \frac{\mu_j}{\sigma_{jk(l)}} \right)$. Now we have,
 
\begin{align*}
\EE \left[ \left(\hat{\mu}^{r}_j(t) - \mu_j \right)^2\right] &= \EE \left[ \left( \frac{1}{W_j(r, t)}\sum_{l=1}^{m} A_l \right)^2 \right] \\
&= \EE\left[ \frac{1}{W_k^2(r,t)} \left( A_m^2 + 2A_m (\sum_{l=1}^{m-1} A_l) + \sum_{l=1}^{m-1} A_l^2 + \sum_{i \neq m,  j\neq m} A_iA_j\right)\right] \\
&\leq \EE \left[ \EE\left[ \frac{1}{W_k^2(r,t)} \left( A_m^2 + 2A_m (\sum_{l=1}^{m-1} A_l) + \sum_{l=1}^{m-1} A_l^2 + \sum_{i \neq m,  j\neq m} A_iA_j\right) \bigg \vert F_{m-1}\right] \right] \\
& \leq \frac{\sigma^2}{m^2} + 0 + \EE \left[ \left( \frac{1}{W_j(r, t)}\sum_{l=1}^{m-1} A_l \right)^2 \right].
\end{align*}
  
We can apply the filtrations $F_{m-l}$ for $l = 2,..., m$ successively to arrive at the result. 
\end{proof}

Now, we can apply Markov on $\left(\hat{\mu}^{r}_j(t) - \mu_j \right)^2$ to conclude that for all $r \in [l(t)]$,

\begin{align}
\label{eq:mmeans}
\PP \left( \vert\hat{\mu}^{r}_k(t) - \mu_k \vert \leq \sigma \sqrt{\frac{4}{m}}\right) \geq \frac{3}{4}. 
\end{align}

Now we will prove Lemma~\ref{lem:mom}. 

\begin{proof}[Proof of Lemma~\ref{lem:mom}]
In light of Eq.~\eqref{eq:mmeans}, the probability that the median is not within distance $ \sigma \sqrt{\frac{4}{m}}$ of $\mu_k$ is bounded as,
\begin{align*}
&\PP\left(\vert\hat{\mu}^m_k(t) - \mu_k \vert >  \sigma \sqrt{\frac{4}{m}}\right) \\
&\leq \PP \left(\mathrm{Bin}(l(t),1/4) > l(t)/2 \right) \leq \delta(t).
\end{align*}
This concludes the proof. 
\end{proof}

Note that we will re-index the experts such that $0 = \Delta_{(1)} \leq \Delta_{(2)} \leq ... \leq \Delta_{(N)}$. Note that throughout this proof $U_k(t),\hat{\mu}_k(t)$ and $s_k(t)$ in Algorithm~\ref{alg:dUCB} are defined as in Equations~\eqref{eq:est2} and~\eqref{eq:ucb2} respectively. Before we proceed let us prove some key lemmas. Now we prove lemmas analogous to Lemmas~\ref{lem:lbound} and \ref{lem:ubound}. 

\begin{lemma}
	\label{lem:ubound_mom} We have the following confidence bound at time $t$,
	\begin{align*}
	\PP \left( U_{k^*}(t) > \mu^* \right) \geq 1 - \frac{1}{t^2}.
	\end{align*}
\end{lemma}

\begin{proof}
		\begin{align*}
	\PP \left( U_{k^*}(t)  \leq \mu^* \right) &= \PP \left(\hat{\mu}_{k^*}(t) \leq \mu^* - \frac{1}{W_k(t)} \sqrt{\frac{64 \sigma^2 \log t}{t}} \right) \\
	& \stackrel{(i)}{\leq} \PP \left(\hat{\mu}_{k^*}(t) \leq  \mu^* - \sigma \sqrt{\frac{64 \log t}{t}}\right) \\
	& \stackrel{(ii)}{\leq} \frac{1}{t^2}.
	\end{align*}
\end{proof}

Here $(i)$ follows from $W_k(t) \geq 1$ and $(ii)$ from Lemma~\ref{lem:mom}.
\begin{lemma}
	\label{lem:lbound_mom} We have the following confidence bound at time $t > \frac{256\sigma^4\log T}{\Delta_k^2}$,
	\begin{align*}
	\PP \left( U_{k}(t) < \mu^* \right) \geq 1 - \frac{1}{t^2}.
	\end{align*}
\end{lemma}

\begin{proof}
	We have the following chain,
	
	\begin{align*}
	\PP \left( U_{k}(t) \geq \mu^* \right) &= \PP \left(\hat{\mu}_k(t) \geq \mu^* - \frac{1}{W_k(t)} \sqrt{\frac{64 \sigma^2 \log t}{t}} \right) \\
	& \stackrel{(i)}{\leq} \PP \left(\hat{\mu}_k(t) > \mu^* - \frac{\Delta_k}{2}\right) \\
	& \stackrel{(ii)}{\leq} \PP \left(\hat{\mu}_k(t) > \mu_k + \frac{\Delta_k}{2} \right) \\
	& \stackrel{(iii)}{\leq} \frac{1}{t^2}.
	\end{align*}
	
	Here, $(i)$ follows from the fact that $\frac{1}{W_k(t)} \leq \sigma$ and $t > \frac{256\sigma^4\log T}{\Delta_k^2}$. $(ii)$ is by definition of $\Delta_k$. Finally the concentration bound in $(iii)$ follows from Lemma~\ref{lem:mom}. 
\end{proof}
	
	Note that Lemma~\ref{lem:lbound_mom} and \ref{lem:ubound_mom} together imply that,
	
	\begin{align}
	\label{eq:badevent2}
	\PP \left(k(t) = k \right) \leq \frac{2}{t^2}
	\end{align}
	for $k \neq k^*$ and $t > \frac{256\sigma^4\log T}{\Delta_k^2}$.

\begin{proof}[Proof of Theorem~\ref{thm:r2}]
	Let $T_k = \frac{256\sigma^4\log T}{\Delta_{(k)}^2}$ for $k = 2,..,N$. The regret of the algorithm can be bounded as,
	\begin{align*}
	&R(T) \leq \Delta_{(N)}T_{N} + \sum_{k = 0}^{N-3} \sum_{t = T_{N - k }}^{T_{N - k -1}} \left(  \PP\left( \mathds{1} \{k(t) \in \{{(1)},...,{(N-k-1)} \} \right)  \times \Delta_{(N-k-1)} + \sum_{i = N-k}^{N} \Delta_{(i)} \PP\left( \mathds{1} \{k(t) = (i) \} \right) \right) \\
	&\leq \Delta_{(N)}T_{N} +  \sum_{k = 0}^{N-3} \left( \Delta_{(N-k-1)} \left(T_{N - k -1} - T_{N - k} \right) + \sum_{t = T_{N - k}}^{T_{N - k -1}} \sum_{i = N-k}^{N} \frac{2\Delta_{(i)}}{t^2}\right)  \numberthis \label{eq:uselater2} \\
	& \leq \sum_{k = 0}^{N-3} \frac{512\sigma^4\log T}{\Delta_{(N -k - 1)}} \left(1 - \frac{\Delta^2_{(N - k -1)}}{\Delta^2_{(N - k)}} \right) + \Delta_{(N)}T_{N} +  \frac{N\pi^2}{3} \left(\sum_{i = 2}^{N} \Delta_{(i)} \right)\\
	&=\frac{512\sigma^4\log T}{\Delta_{(N)}} + \frac{N\pi^2}{3} \left(\sum_{i = 2}^{N} \Delta_{(i)} \right) +\sum_{k = 2}^{N-1}\frac{512\sigma^4\log T}{\Delta_{(k)}} \left(1 - \frac{\Delta^2_{(k)}}{\Delta^2_{(k+1)}} \right).
	\end{align*}

\end{proof}

\section{Instance Dependent Terms}
\label{sec:regret}

In this section we devote our attention to the instance dependent terms in Theorems~\ref{thm:r1} and \ref{thm:r2}. We will first prove Corollary~\ref{cor:simple}. 

\begin{proof}[Proof of Corollary~\ref{cor:simple}]
	We will prove the two statements about the two estimators separately,
	
	$(i)$  Going back to Lemma~\ref{lem:lbound} in the proof of Theorem~\ref{thm:r1}, we get that,
	\begin{align*}
	\PP(U_k(t) < \mu^*) \geq 1 - \frac{1}{t^2}
	\end{align*}
	when $t > \frac{2304M^4 \log ^2 \left( \frac{96M}{\Delta_{(2)}\gamma}\right) \log T}{\Delta_{(k)}^2 \gamma^2}$. This simply follows from the fact that $\Delta_{(2)}$ is the smallest gap. Therefore, the chain leading to Eq.~\eqref{eq:uselater} follows with the new definition of $T_k = \frac{2304M^4 \log ^2 \left( \frac{96M}{\Delta_{(2)}\gamma}\right)}{\Delta_{(k)}^2 \gamma^2}$. Hence, the regret of Algorithm~\ref{alg:dUCB} under estimator~\eqref{eq:est1} is bounded as follows:
	\begin{align*}
	&R(T) \leq  \frac{2304M^4 \log ^2 \left( \frac{96M}{\Delta_{(2)}\gamma}\right) \log T}{\Delta_{(k)} \gamma^2} \left(1  + \sum_{k = 2}^{N-1} \left(1 - \frac{\Delta_{(k)}^2}{\Delta_{(k+1)}^2} \right) \right) + \frac{\pi^2}{3} \left(\sum_{i = 2}^{N} \Delta_{(i)} \right)\\
	& =  \frac{2304\lambda(\pmb{\mu})M^4\log^2 (96M/\gamma\Delta_{(2)} ) \log T}{\Delta_{(2)}} + \frac{\pi^2}{3} \left(\sum_{i = 2}^{N} \Delta_{(i)} \right). \numberthis \label{eq:r12}
	\end{align*}
	
	We can analyze the same terms in an alternate manner. From Eq.~\eqref{eq:uselater} in the proof of Theorem~\ref{thm:r1}, it follows that the regret of Algorithm~\ref{alg:dUCB} under the clipped estimator is bounded by,
	\begin{align*}
	R(T) \leq \Delta_{(N)}T_{2} + \frac{\pi^2}{3} \left( \sum_{i} \Delta_{(i)}\right).
	\end{align*}
	Using the definition of $T_{2}$ in~\eqref{eq:uselater} we obtain:
	\begin{align*}
	&R(T) \leq \frac{2304M^4\log^2 (96M/\gamma\Delta_{(2)}) \log T}{\Delta_{(2)}} \times \frac{1}{\Delta_{(2)}}. 
	\end{align*}
	Combining the above equation with~\eqref{eq:r12} we get the desired result.

$(ii)$ Theorem~\ref{thm:r2} immediately implies that 

\[R(T) \leq \frac{256\lambda(\pmb{\mu})\sigma^4  \log T}{\Delta_{(2)}} + \frac{\pi^2}{3} \left(\sum_{i = 2}^{N} \Delta_{(i)} \right) \label{eq:r21},\] for the median of means estimator. 

We can alternately analyze the regret as follows. From Eq.~\eqref{eq:uselater2} in the proof of Theorem~\ref{thm:r2}, it follows that the regret of Algorithm~\ref{alg:dUCB} under the median of means estimator is bounded by,
\begin{align*}
R(T) \leq \Delta_{(N)}T_{2} + \frac{\pi^2}{3} \left( \sum_{i} \Delta_{(i)}\right).
\end{align*}
 Using the definition of $T_{2}$ in~\eqref{eq:uselater2} we obtain:
\begin{align*}
&R(T) \leq \frac{256\sigma^4 \log T}{\Delta_{(2)}} \times \frac{1}{\Delta_{(2)}}. 
\end{align*}

Combining the equations above we get the desired result. 

\end{proof}

Now we will work under the assumption that the gaps in the means of the experts are generated according to the generative model in Corollary~\ref{lem:mu}. 

\begin{proof}[Proof of Corollary~\ref{lem:mu}]
	In light of Corollary~\ref{cor:simple}, we just need to prove that $\EE_{p_{\Delta}}[\lambda(\pmb{\mu})] = O(\log N)$. 
	
	Now, we will assume that $\{\Delta_{(i)}\}$ for $i = 3,...,N$, are order statistics of $N-2$ i.i.d uniform r.vs over the interval $[\Delta_{(2)},1]$. 
	
	Note that by Jensen's we have the following:
	\begin{align}
	\label{eq:jensen}
	1 - \EE \left[ \frac{\Delta_{(k)}^2}{\Delta_{(k+1)}^2} \right] \leq 1 - \EE \left[  \frac{\Delta_{(k)}}{\Delta_{(k+1)}} \right]^2. 
	\end{align}

	Let $X = \Delta_{(k)}$ and $Y = \Delta_{(k+1)}$ for $k \geq 3$. The joint pdf of $X,Y$ is given by,
	
	\begin{align*}
	&f(x,y) = \frac{(N-2)!}{(k-1)!(N-3-k)!} \left(\frac{x - \Delta_{(2)}}{1 - \Delta_{(2)} } \right)^{k-1} \times \left(1 - \frac{y - \Delta_{(2)}}{1 - \Delta_{(2)} } \right)^{N -3 -k} \frac{1}{(1 - \Delta_{(2)} )^2}.
	\end{align*}
	
	Therefore, we have the following chain:
	\begin{align*}
	&\EE \left[ \frac{X}{Y}\right] = \int_{y = \Delta_{(2)} }^{1} \int_{x = \Delta_{(2)}}^{y} \frac{x}{y} \frac{(N-2)!}{(k-1)!(N-3-k)!} \times \left(\frac{x - \Delta_{(2)}}{1 - \Delta_{(2)} } \right)^{k-1} \left(1 - \frac{y - \Delta_{(2)}}{1 - \Delta_{(2)} } \right)^{N -3 -k} \frac{1}{(1 - \Delta_{(2)} )^2} dx dy \\
	&= \int_{b = 0 }^{1} \int_{a = 0}^{b} \frac{(1 - \Delta_{(2)})a+ \Delta_{(2)}}{(1 - \Delta_{(2)})b+ \Delta_{(2)}} \frac{(N-2)!}{(k-1)!(N-3-k)!} \times \left(a \right)^{k-1} \left(1 -b \right)^{N -3 -k}  da db \\ 
	&\geq \int_{b = 0 }^{1} \int_{a = 0}^{b} \frac{a}{b} \frac{(N-2)!}{(k-1)!(N-3-k)!} \times \left(a \right)^{k-1} \left(1 -b \right)^{N -3 -k}  da db \\
	& = \frac{k}{k+1}. 
	\end{align*}
	
	Combining this with Eq~\eqref{eq:jensen} yields,
	\begin{align*}
	&\EE_{p_{\Delta}} \left[ \sum_{k = 2}^{N-1} \left(1 - \frac{\Delta_{(k)}^2}{\Delta_{(k+1)}^2} \right) \right] \\
	&\leq 1 + \sum_{k = 3}^{N-1} \left(1 -  \frac{k^2}{(k+1)^2} \right) \\
	&= 1 + \sum_{k = 3}^{N-1} \left(\frac{2k+1}{(k+1)^2} \right) \\
	&\leq 1 + \sum_{k = 3}^{N-1} \frac{2}{k+1} \numberthis \label{eq:deltabound}\\
	&\leq 1 + 2\log N.
	\end{align*}
\end{proof}

\section{More on Empirical Results}
\label{sec:more_emp}

In this section we provide more details about our empirical results under the following sub-headings. 

\begin{figure}[h]
	\centering
	\includegraphics[width=0.4\linewidth]{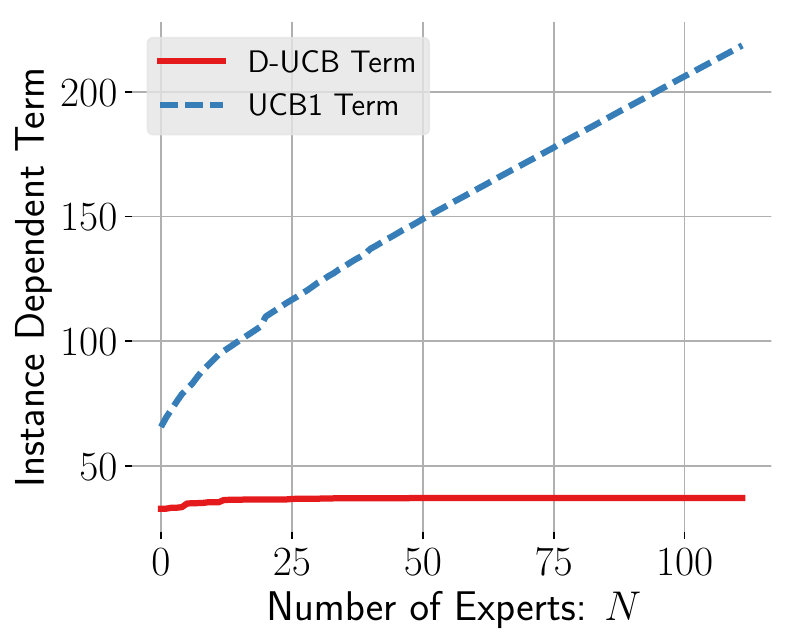}
	\caption{\small We plot the instance-dependent terms from D-UCB bounds (the term in Theorem~\ref{thm:r2} involving the gaps~\eqref{eq:instance1}) and that of UCB-1 bounds~\eqref{eq:instance2} as the number of experts grows in the stream analytics dataset. It can be observed that the instance dependent term from D-UCB grows at a much slower pace with the number of experts, and in fact stops increasing after a certain point. }
	\label{fig:terms}
\end{figure}

{\bf Training of Stochastic Experts: } In Algorithm~\ref{alg:batched}, new experts are added before starting a new batch. These stochastic experts are classifying functions trained using cost-sensitive classification oracles on data observed so far, which uses the ideas in~\cite{beygelzimer2009offset}. The key idea is to reduce the cost-sensitive classification problem into importance weighted classification, which can be solved using binary classifiers by providing weights to each samples. Suppose a context $x$ is observed and Algorithm~\ref{alg:batched} chooses an expert $\pi_i$ and draws an arm $a$ from the conditional distribution $\pi_i(V \vert x)$. Suppose the reward observed is $r(a)$. Then the training sample $(x,a)$ with a sample weight of $r(a)/\pi_i(a \vert x)$ is added to the dataset for training the next batch of experts. It has been shown that this importance weighing yields \textit{good} classification experts. These classifiers can provide confidence scores for arms, given a context and hence can serve as stochastic experts. $4$ different experts are added at the beginning of each batch, out of which three are trained by XgBoost as base-classifier while one is trained by logistic regression. Diversity is maintained among the experts added by training them on bootstrapped versions of the data observed so far, and also through selecting different hyper-parameters. Note that the parameter selection scheme is not tuned per dataset,  but is held fixed for all three datasets. 

{\bf Estimating Divergence Parameters: } Both our divergence metrics $M_{ij}$'s and $\sigma_{ij}$'s can be estimated from data observed so far, during a run of Algorithm~\ref{alg:batched}. These divergences do not depend on the arm chosen, but only on the context distribution and the conditional distributions encoded by the expert. Therefore, they can be easily estimated from data observed. Suppose, $n$ contexts have been observed so far $\{x_1,...,x_n \}$. We are interested in estimating $\sigma_{ij}$ that is the chi-square divergence between $\pi_i$ and $\pi_j$. An estimator for this would be the empirical mean $(1/n) \times \sum_{k = 1}^{n} D_{f_2} (\pi_i(.\vert x_k) \Vert \pi_j(.\vert x_k) )$. Note that the distribution over the arms $\pi_j(.\vert x_k)$ is nothing but the confidence scores observed through evaluation of the classifying oracle $\pi_j$ on the features/context $x_i$. In order to be robust, we use the median of means estimator instead of the simple empirical mean for estimating the divergences. 

{\bf Empirical Analysis of Instance Dependent terms: } In this section we empirically validate that our instance dependent terms in Theorem~\ref{thm:r1} and \ref{thm:r2} are indeed much smaller compared to corresponding terms in the UCB-1~\cite{auer2002using} regret bounds, even in real problem where our generative assumptions do not hold. In order to showcase this, we plot the instance-dependent term in Theorem~\ref{thm:r2} which is given by, \[ \sum_{k = 2}^{N-1} \frac{1  }{\Delta_{(k)}} \left(1 - \frac{\Delta_{(k)}^2}{\Delta_{(k+1)}^2} \right) + \frac{1}{\Delta_{(N)}} \numberthis \label{eq:instance1}\] along with the corresponding term in UCB-1 bounds given by, \[ \sum_{k = 2}^{N} \frac{1}{\Delta_{(k)}}\numberthis \label{eq:instance2}, \] as the number of stochastic experts grow in the stream dataset experiments in Section~\ref{sec:sims}. The true means of the experts have been estimated in hindsight over the whole dataset. The plot is shown in Fig.~\ref{fig:terms}. It can be observed that the term in the bounds of D-UCB grows at a much slower pace, and in fact stops increasing with the number of experts after a certain point.

\end{document}